\documentclass{article} 

\usepackage{iclr2024_conference,times}

\usepackage{amsmath,amsfonts,bm}

\def\eqref#1{equation~\ref{#1}}

\def\1{\bm{1}}

\DeclareMathAlphabet{\mathsfit}{\encodingdefault}{\sfdefault}{m}{sl}
\SetMathAlphabet{\mathsfit}{bold}{\encodingdefault}{\sfdefault}{bx}{n}

\usepackage[colorlinks=true,citecolor=teal]{hyperref}
\usepackage{url}
\usepackage{booktabs}       
\usepackage{amsfonts}       
\usepackage{nicefrac}       
\usepackage{microtype}      
\usepackage{xcolor}         
\usepackage{cancel}

\usepackage{graphicx}
\usepackage{subfigure}
\usepackage{algorithm} 
\usepackage{algpseudocode}
\usepackage{pifont}
\usepackage{rotating}
\usepackage{float}

\usepackage{colortbl}  
\usepackage{xcolor}
\usepackage{array}   

\usepackage{wrapfig}
\usepackage{caption}
\usepackage{amsmath}
\usepackage{color}

\usepackage{listings}
\usepackage{xcolor}

\usepackage{multirow}
\usepackage[frozencache,cachedir=.]{minted}
\usepackage{bbm, dsfont}

\usepackage{amsmath,amssymb,amsthm}
\newtheorem{theorem}{Theorem}[section]

\usepackage{titletoc}
\newcommand\DoToC{%
  \startcontents
  \printcontents{}{1}{\hrulefill\vskip0pt}
  \vskip0pt \noindent\hrulefill
}

\title{A hard-to-beat baseline for training-free CLIP-based adaptation}

\author{Zhengbo Wang$^{1,2}$\quad Jian Liang$^{2,3}$\thanks{Corresponding author.}\quad Lijun Sheng$^{1,2}$\quad Ran He$^{2,3}$\quad Zilei Wang$^{1}$\quad Tieniu Tan$^{2,4}$ \\
\\
$^1$ University of Science and Technology of China \\
$^2$ CRIPAC \& MAIS, Institute of Automation, Chinese Academy of Sciences (CASIA)\\
$^3$ School of Artificial Intelligence, University of Chinese Academy of Sciences $^4$ Nanjing University \\
\texttt{zhengbowang@mail.ustc.edu.cn, liangjian92@gmail.com} \\
}

\iclrfinalcopy 
\begin{document}

\maketitle

\begin{abstract}
Contrastive Language-Image Pretraining (CLIP) has gained popularity for its remarkable zero-shot capacity.
Recent research has focused on developing efficient fine-tuning methods, such as prompt learning and adapter, to enhance CLIP's performance in downstream tasks.
However, these methods still require additional training time and computational resources, which is undesirable for devices with limited resources.
In this paper, we revisit a classical algorithm, Gaussian Discriminant Analysis (GDA), and apply it to the downstream classification of CLIP.
Typically, GDA assumes that features of each class follow Gaussian distributions with identical covariance.
By leveraging Bayes' formula, the classifier can be expressed in terms of the class means and covariance, which can be estimated from the data without the need for training.
To integrate knowledge from both visual and textual modalities, we ensemble it with the original zero-shot classifier within CLIP.
Extensive results on 17 datasets validate that our method surpasses or achieves comparable results with state-of-the-art methods on few-shot classification, imbalanced learning, and out-of-distribution generalization.
In addition, we extend our method to base-to-new generalization and unsupervised learning, once again demonstrating its superiority over competing approaches.
Our code is publicly available at \url{https://github.com/mrflogs/ICLR24}.
\end{abstract}

\section{Introduction}
Contrastive Language-Image Pretraining, known as CLIP~\citep{radford2021learning}, has attracted considerable attention in recent years as a powerful method for aligning vision and language representations.
By leveraging a massive dataset of 400 million web-scale image-text pairs, CLIP learns to encode images and text into a shared semantic space using vision and language encoders, respectively.
This shared semantic space facilitates the comparison of similarities and differences between images and texts.
One remarkable feature of CLIP is its zero-shot ability to perform image classification without any additional training on the target classes.
This is accomplished by using the language encoder to generate classifier weights based on a simple prompt, such as ``a photo of a \{class\}".
By inputting different class names into the prompt,
we can generate the corresponding weights for classification, thereby enabling the classification of images into a broad range of categories without training.

Despite its powerful zero-shot capability, recent works~\citep{zhou2022learning, zhang2022tip, zhou2022conditional, lu2022prompt, gao2021clip} have focused on designing efficient fine-tuning methods for downstream classification tasks. 
These methods have achieved significant improvements compared to Zero-Shot CLIP~\citep{radford2021learning}, even when trained on few-shot datasets and optimizing extremely small numbers of parameters.
For instance, CoOp~\citep{zhou2022learning} proposes to learn a set of global text prompts for the pre-trained CLIP~\citep{radford2021learning} on downstream tasks, which achieves a 15\% improvement compared to Zero-Shot CLIP~\citep{radford2021learning} with only 16 samples per class by fine-tuning a mere 16k parameters. 
While these methods are efficient and yield satisfactory results in downstream tasks, they demand extra computational resources to learn new parameters. 
This can be inconvenient, especially on devices with limited resources. 

We aim to develop a method that not only eliminates the need for additional training, similar to Zero-Shot CLIP but also attains comparable or even better results than these training-required methods.
To achieve this, we revisit a classical algorithm, Gaussian Discriminant Analysis (GDA)~\citep{bishop2006pattern}, and apply it to the downstream classification tasks of CLIP.
Typically, GDA assumes that features from different classes follow Gaussian distributions with identical covariance.
By leveraging Bayes' formula, the classification probability, $p(y|x)$, can be expressed as a softmax of a linear function, where the weight and bias are determined by the mean vectors and covariance.
Based on this, we can compute the mean vectors and covariance from the training dataset to construct the classifier, thereby eliminating the need for any additional training process such as SGD.
To fully utilize the knowledge in pretrained CLIP, we ensemble it with the CLIP zero-shot classifier, integrating the knowledge from both visual and textual modalities. 

To further demonstrate its efficacy, we develop two variants of our method for base-to-new generalization and unsupervised learning, where our method can not be directly applied.
In the base-to-new generalization, the model is trained on a base dataset and then adapted to a new dataset with different classes.
Based on the observation that similar samples have similar statistical information~\citep{yangfree}, we propose using the K-Nearest-Neighbor algorithm to synthesize data for new classes, and we obtain the classifier for new classes using these synthesized data.
For unsupervised learning, where no labeled data are available for training, the mean and covariance cannot be directly obtained from data.
Based on the Gaussian assumption in GDA, the unlabeled data then follow a Gaussian mixture distribution.
We straightforwardly employ the EM algorithm to estimate its mean and covariance.
In both scenarios, we endeavor to keep the modifications relatively straightforward to avoid introducing additional complexity that might impact the assessment of our approach.
Nevertheless, these two simple variants still achieve comparable performance with previous complex state-of-the-art methods.

We conduct extensive experiments on 17 widely adopted datasets to evaluate the effectiveness of our method. 
Despite the simplicity of our approach, we demonstrate that our method serves as a hard-to-beat baseline.
In few-shot classification, our method surpasses state-of-the-art training-free methods by a margin of 2.82\% on average over 11 datasets, and it achieves comparable performance with training-required methods (76.05\% vs. 75.83\%). 
For imbalanced learning, our method outperforms previous state-of-the-art methods, even if they are fully fine-tuned.
The two variants of base-to-new generalization and unsupervised learning achieve comparable performance with previous methods.
These results underscore the effectiveness of our method.
\section{Related Work}
\textbf{Vision-Language Models.} In recent years, vision-language models (VLMs) have become a new paradigm for foundational models that aim to bridge the gap between the modalities of vision and language.
These models are trained on large-scale image-text datasets, which endows them with powerful transferable abilities such as zero-shot learning, few-shot adaptation, and in-context learning~\citep{radford2021learning, kim2021vilt, lu2019vilbert, su2019vl, jia2021scaling}.
Moreover, they exhibit strong open-world capabilities and have been successfully applied to recognize open-world concepts, including zero-shot learning~\citep{radford2021learning, jia2021scaling}, open-world segmentation~\citep{xu2021, ding2022decoupling}, and open-world detection~\citep{joseph2021towards, gu2021open, gupta2022ow}.
Contrastive-based vision-language pre-training has become the mainstream approach in this field.
These methods, including CLIP~\citep{radford2021learning} and ALIGN~\citep{jia2021scaling}, are trained on large-scale web-based noisy image-text pairs. 
They employ a language encoder and a vision encoder to encode the texts and images, respectively, and learn to align their representations through contrastive loss. 
We utilize CLIP~\citep{radford2021learning} in this work.

\textbf{Efficient Fine-tuning for VLMs.}
Recent works~\citep{zhou2022conditional, zhou2022learning, zhang2022tip, gao2021clip, lu2022prompt, chen2022prompt, guo2022calip, udandarao2022sus, huang2022unsupervised, wang2023exploring, wang2023improving} focus on developing efficient fine-tuning methods for large pre-trained vision-language models that can be used in downstream tasks due to their large model size.
These methods aim to achieve the maximum performance gain by fine-tuning the minimum number of model parameters on few-shot downstream datasets.
For instance, CoOp~\citep{zhou2022learning} proposes to learn global text prompts for downstream tasks through back-propagation on few-shot datasets.
Meanwhile, CLIP-Adapter~\citep{gao2021clip} proposes to learn a visual and a textual adapter to refine the original representations of the vision-language models.
Despite being efficient and achieving significant improvements, previous works~\citep{zhou2022conditional, wang2023improving} have found that these methods tend to exhibit poor generalization when confronted with new classes.
In response to this limitation, CoCoOp~\citep{zhou2022conditional} proposes a solution by incorporating visual information into the text prompt for regularization, leading to improved base-to-new generalization performance. 
Besides, other works~\citep{wang2023exploring, huang2022unsupervised, shu2022tpt} attempt to adapt CLIP~\citep{radford2021learning} for imbalanced learning and unsupervised learning scenarios.
While these approaches have achieved satisfactory results, they still require additional training processes to learn the newly introduced parameters, which is undesirable for devices with limited resources.

\textbf{High-dimensional Covariance Estimation.}
Our method involves estimating the covariance or precision matrices in high-dimensional space, which can be challenging, particularly when data is limited.
Typically, the covariance matrix is estimated using the Maximum Likelihood Estimator~(MLE), but this is not a reliable estimator of the eigenvalues of the covariance matrix. 
As a result, the precision matrix obtained from the inversion of the estimated covariance matrix may not be accurate. 
In some cases, it may even be impossible to invert the empirical covariance matrix due to numerical issues. 
To address this problem, previous works have proposed shrinkage methods~\citep{ledoit2004well, chen2010shrinkage, kubokawa2008estimation, efron1976multivariate}.
For instance, 
Ledoit-Wolf shrinkage~\citep{ledoit2004well} provides a formula to calculate the optimal shrinkage coefficient that minimizes the MSE between the estimated and actual covariance matrices.
Similarly, OAS~\citep{chen2010shrinkage} presents a formula that aims to select a shrinkage coefficient that results in a lower MSE than the one given by Ledoit-Wolf shrinkage~\citep{ledoit2004well}.
However, these approaches require additional optimization processes to estimate the covariance and precision matrices. In our work, we use the empirical Bayes ridge-type estimator~\citep{kubokawa2008estimation}, which does not require any training, to address this challenge.

\section{Method}
\subsection{Gaussian Discriminant Analysis for CLIP Adaptation}
\textbf{Gaussian Discriminant Analysis.}
We revisit a traditional probabilistic generative model~\citep{bishop2006pattern}, Gaussian Discriminant Analysis (GDA), for training-free CLIP-based adaptation, whose classifier can be derived by making an assumption about the data distribution of each class.
The parameters of the classifier can be obtained from the statistical information of the data without the need for explicit training.
By applying Bayes' formula, the classification probability can be expressed as the function of the data distribution and its prior probability:
\begin{equation}
\label{eq:bayes}
    p(y=i|x) = \frac{p(x|y=i)p(y=i)}{\sum_{j=1}^Kp(x|y=j)p(y=j)} = \frac{\exp(f_i(x))}{\sum_{j=1}^K\exp(f_j(x))},
\end{equation}
where $i=1,2,\dots,K$ for $K$-class classification tasks, $x\in\mathbb{R}^D$ is the visual feature, and we normalize Eq.~(\ref{eq:bayes}) using the softmax function. 
And the logit function is $f_i(x) = \log (p(x|y=i)p(y=i)), i=1,2,\dots,K$.
Therefore, by assuming the data distribution of each class and their prior distribution, we can obtain the classifier.
In GDA~\citep{bishop2006pattern}, the features are typically assumed to follow the Gaussian distributions with identical covariance, i.e., $(X|Y=i)\sim \mathcal{N}(\mu_i, \Sigma)$ for $i=1,2,..,K$.
We substitute this assumption into Eq.~(\ref{eq:bayes}), which then can be expressed as follows:
\begin{equation}
    \label{eq:softmax}
    p(y=i|x) = \frac{\exp(\mu_i^T\Sigma^{-1}x - \frac{1}{2}\mu_i^T\Sigma^{-1}\mu_i + \log p_i)}{\sum_{j=1}^K\exp(\mu_j^T\Sigma^{-1}x - \frac{1}{2}\mu_j^T\Sigma^{-1}\mu_j + \log p_j)},
\end{equation}
where $p_i=p(y=i)=1/K, i=1,2,...,K$ is the prior probability of the corresponding class, which is assumed to be uniform.
And thus, the logit $f_i(x) = \mu_i^T\Sigma^{-1}x - \frac{1}{2}\mu_i^T\Sigma^{-1}\mu_i + \log p_i$.
Thus, the weight $W\in\mathbb{R}^{K\times D}$ and the bias $b\in\mathbb{R}^{K}$ for the classifier are as follows:
\begin{equation}
    \label{eq:solution}
    w_i = \Sigma^{-1}\mu_i, \quad
    b_i = \log p_i - \frac{1}{2}\mu_i^T\Sigma^{-1}\mu_i,
\end{equation}
for $i=1,2,...,K$.
Later, we estimate the mean for each class and the precision matrix using the training data and subsequently obtain the corresponding weight and bias for the linear classifier.

\begin{figure}[t]
    \centering
    \includegraphics[width=1.0\textwidth]{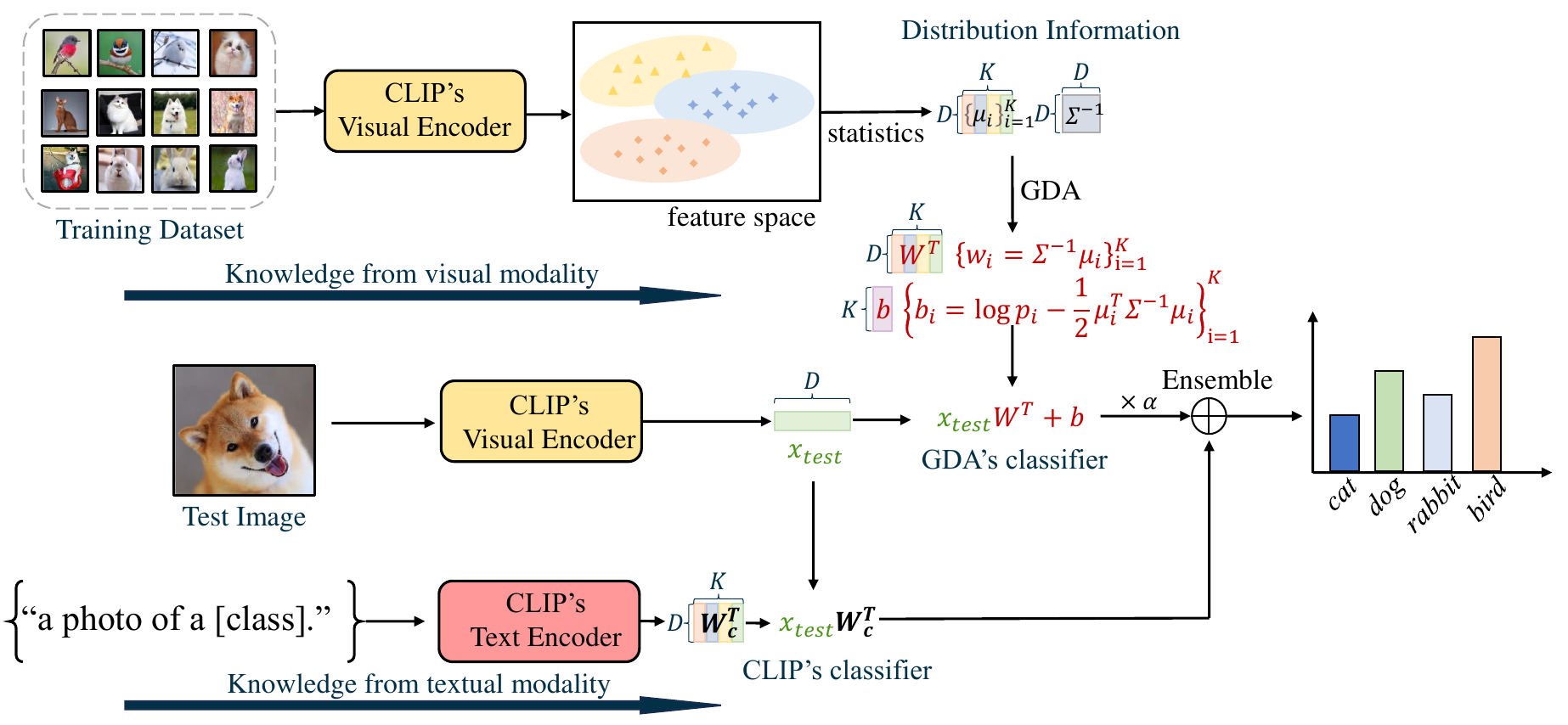}
    \caption{
    \textbf{The overview of our training-free method.}
    In our method, we begin by extracting visual features from the training dataset using the CLIP visual encoder. 
    Next, we compute the mean vectors for each class and the shared precision matrix (inverse covariance) using Eq.~(\ref{eq:precisionmatrix}).
    Through the Gaussian Discriminate Analysis (GDA), the weight and bias of the classifier can be expressed in terms of the mean vectors and the precision matrix, which can be derived from Eq.~(\ref{eq:solution}) (\textcolor{red}{the red formula in the figure}).
    Finally, we enhance our method by ensembling the GDA classifier and the CLIP's zero-shot classifier, integrating the knowledge from visual and textual modalities.
  }
    \label{fig:architecture}
    \vspace{-15pt}
\end{figure}

\textbf{Parameter Estimation.}
We estimate the mean vectors using the empirical mean $\hat{\mu}_k = \sum_{j=1}^N\mathbb{I}_{(y_j=k)}x_j / \sum_{j=1}^N\mathbb{I}_{(y_j=k)}$. 
However, in high-dimensional spaces, estimating the precision matrix is a challenging task due to the inverse of the empirical covariance matrix being a biased estimator of the precision matrix, and it may be impossible to invert due to numerical issues. 
To solve this, we utilize shrinkage methods to estimate the precision matrix.
To avoid introducing additional computations, we use the empirical Bayes ridge-type estimator~\citep{kubokawa2008estimation} to estimate the precision matrix:
\begin{equation}
\label{eq:precisionmatrix}
    \widehat{\Sigma^{-1}} = D((N-1)\hat{\Sigma} + tr(\hat{\Sigma})I_D)^{-1},
\end{equation}
where $N$ is the number of samples, $D$ is the dimension of the features, and $\hat{\Sigma}$ represents the empirical covariance. 
Once the parameter estimation is completed, we can input it into Eq.~(\ref{eq:solution}) to obtain the weight and bias of the classifier.

Besides the knowledge extracted from visual modality, the prior knowledge of text modality in pre-trained CLIP is calculated by $x_{test}W_c^T$, where $W_c$ is the weights of CLIP's classifier generated from the text encoder by inputting a predefined prompt, such as ``a photo of a \{class\}".
For simplicity, we integrate the knowledge from visual and text modalities by mixing the predictions.
Therefore, the output logits of the test image are then calculated as:
\begin{equation}
    \label{eq:ensemble}
    logits = x_{test}W_c^T + \alpha (x_{test}W^T + b),
\end{equation}
where $x_{test}$ is the visual feature of test image, and $\alpha$ is a hyper-parameter.

\subsection{Extension to other scenarios}
We further extend our method to base-to-new generalization and unsupervised learning, where our method cannot directly apply, to illustrate the generalization of our method.
In order to maintain the simplicity of the method and avoid introducing additional complexities that could impact the assessment of our approach, we only perform straightforward modifications in these two scenarios.

\textbf{Extension to Base-to-new Generalization.}
For CLIP base-to-new generalization, the model is trained on the base dataset and tested on a new dataset with unseen classes.
However, our method cannot be directly implemented in the scenario where data for the new classes is unavailable.
Based on the observation that similar samples have similar statistical information~\citep{yangfree}, we propose that our method can be extended to new classes using the KNN algorithm.
To achieve this, we utilize text embeddings of the new classes to query the training set and select the k nearest neighbors as the synthesized labeled data. 
The process is defined as follows:
\begin{equation}
    \Tilde{\mathcal{D}}_{new} = \bigcup_{i=K+1}^M\{(x, i)|x\in NN^k(t_i, \mathcal{D})\}
\end{equation}
where $i=K+1, ..., M$ denotes the new classes, $t_i$ denotes its text embedding, and ${NN}^k(*,\mathcal{D})$ denotes the k-nearest neighbors of training set $\mathcal{D}$.
The classifier is then produced utilizing Eq.~(\ref{eq:solution}).

\textbf{Extension to Unsupervised Learning.}
In the unsupervised learning scenario, we only have the unlabeled data $\{x_i\}_{i=1}^N$.
Based on the Gaussian assumption in GDA, the unsupervised data $\{x_i\}_{i=1}^N$ follow Gaussian mixture distribution.
In order to maintain the simplicity of our method, we directly employ the EM algorithm for estimating the means and covariance matrix.
To begin, we initialize the mean vectors and covariance using the zero-shot classifier, assuming equal priors for each Gaussian distribution.
In the E-step, we calculate the probability of the unlabeled data $\{x_i\}_{i=1}^N$ as follows:
\begin{equation}
    \gamma_{ik} = \frac{\exp(f_k(x))}{\Sigma_{j=1}^K\exp{(f_j(x_i))}},
\end{equation}
for the unlabeled data $\{x_i\}_{i=1}^N$, and $f$ is the logit function using Eq.~(\ref{eq:ensemble}).
Moving on to the M-step, we update the mean vectors and covariance matrix using the following formulas:
\begin{equation}
    \mu_k = \frac{\sum_{i=1}^N\gamma_{ik}x_i}{\sum_{i=1}^N\gamma_{ik}}, \quad
    \Sigma =\frac{1}{K}\sum_{k=1}^K\frac{\sum_{i=1}^N\gamma_{ik}(x_i - \mu_k)(x_i - \mu_k)^T}{\sum_{i=1}^N\gamma_{ik}}.
\end{equation}
Subsequently, we update the classifier using Eq.~(\ref{eq:solution}) and repeat the EM process until convergence.
\section{Experiments}
\subsection{Setup}
\label{subsec:setup}
\textbf{Dataset.}
According to previous works~\citep{radford2021learning,zhou2022conditional, zhou2022learning, zhang2022tip}, we select 11 publicly available image classification datasets to assess the effectiveness of CLIP few-shot classification, base-to-new generalization, and unsupervised learning.
These datasets cover a range of image recognition tasks, including generic object recognition with ImageNet~\citep{deng2009imagenet} and Caltech101~\citep{fei2004learning}, fine-grained image recognition with OxfordPets~\citep{parkhi2012cats}, StanfordCars~\citep{krause20133d}, Flowers102~\citep{nilsback2008automated}, Food101~\citep{bossard2014food} and FGVCAircraft~\citep{maji2013fine}, satellite image classification with EuroSAT~\citep{helber2019eurosat}, action classification with UCF101~\citep{soomro2012ucf101}, texture classification with DTD~\citep{cimpoi2014describing}, and scene recognition with SUN397~\citep{xiao2010sun}.  
Additionally, we also select 4 datasets, ImageNetV2~\citep{recht2019imagenet}, ImageNet-Sketch~\citep{wang2019learning}, ImageNet-A~\citep{hendrycks2021natural}, and ImageNet-R~\citep{hendrycks2021many}, to evaluate the out-of-distribution generalization.
Moreover, we adopt 2 commonly used imbalanced datasets, ImageNet-LT~\citep{liu2019large} and Places-LT~\citep{zhou2017places}, for CLIP long-tailed classification. 

\textbf{Training Details.}
To align with previous works~\citep{zhou2022learning,zhou2022conditional,zhang2022tip}, 
we utilize ResNet-50~\citep{he2016deep} as the visual encoder of CLIP for few-shot classification by default.
Similarly, following the previous work~\citep{wang2023exploring}, we choose ResNet-101 as the visual encoder of CLIP for imbalanced learning.
To evaluate the model's base-to-new generalization and out-of-distribution generalization performance, we followed CoCoOp~\citep{zhou2022conditional} and adopted ViT-B/16-based CLIP~\citep{radford2021learning}.
We follow CLIP~\citep{radford2021learning} to adopt prompt ensembling on ImageNet and use a single Zero-Shot CLIP on the other 10 datasets.
The hyperparameter $\alpha$, which is used to ensemble the classifiers, is searched in the validation set with values ranging from $0.0001$ to $100.0$, and this value is kept constant for new class data.
And the $k$ for the KNN algorithm to synthesize the new class dataset is set to 64.
All experiments are conducted on a single NVIDIA GeForce RTX 3090.
To obtain a reliable estimate of model performance, we conduct three runs with different random seeds and averaged the results.

\textbf{Evaluation Protocol.}
For the few-shot classification, we adhere to the evaluation protocol proposed by CLIP~\citep{radford2021learning}.
Specifically, we randomly select 1, 2, 4, 8, or 16 instances per class to form the few-shot datasets. 
Subsequently, we train our model on the few-shot datasets and evaluate its performance on the full test dataset.
For base-to-new generalization, we adopt the standard protocol proposed by CoCoOp~\citep{zhou2022conditional}.
On each dataset, we split the classes equally into two groups, one as base classes and the other as the new classes.
The method is trained using only the 16-shot base classes while the evaluation is conducted on base and new classes separately to test generalization ability. 
Regarding unsupervised learning, we adhere to the evaluation protocol described by UPL~\citep{huang2022unsupervised}.
For imbalanced learning, we split the classes in each benchmark into three groups based on the number of available images per class. 
These groups are referred to as Many-shot (with more than 100 images), Medium-shot (with 20 to 100 images), and Few-shot (with less than 20 images).
We report the accuracy of each group and the macro F1 score.
\vspace{-5pt}
\subsection{Results on Few-Shot Classification}
\vspace{-5pt}
\begin{figure}[!htbp]
    \centering
    \begin{minipage}{0.24\textwidth}
        \centering
        \includegraphics[width=1.13\textwidth]{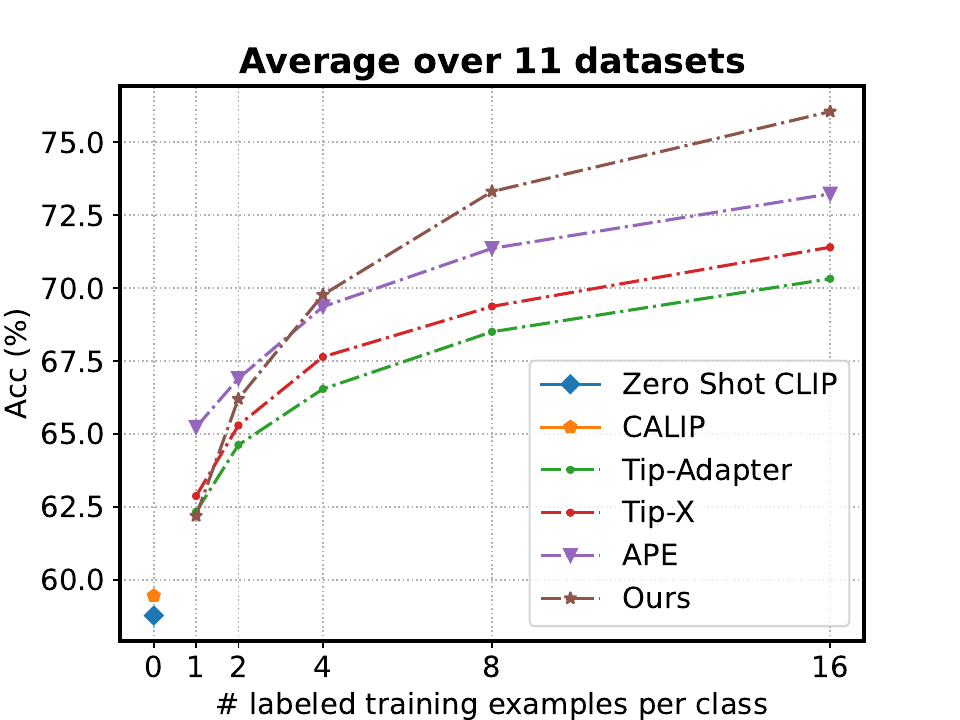}
    \end{minipage}
    \begin{minipage}{0.24\textwidth}
        \centering
        \includegraphics[width=1.13\textwidth]{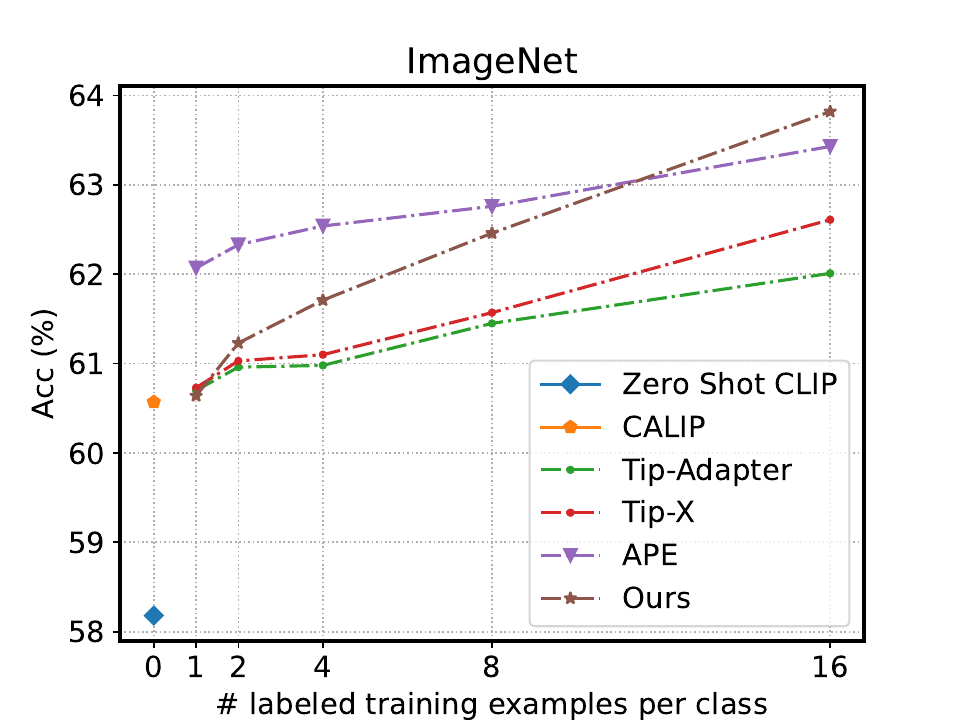}
    \end{minipage}
    \begin{minipage}{0.24\textwidth}
        \centering
        \includegraphics[width=1.13\textwidth]{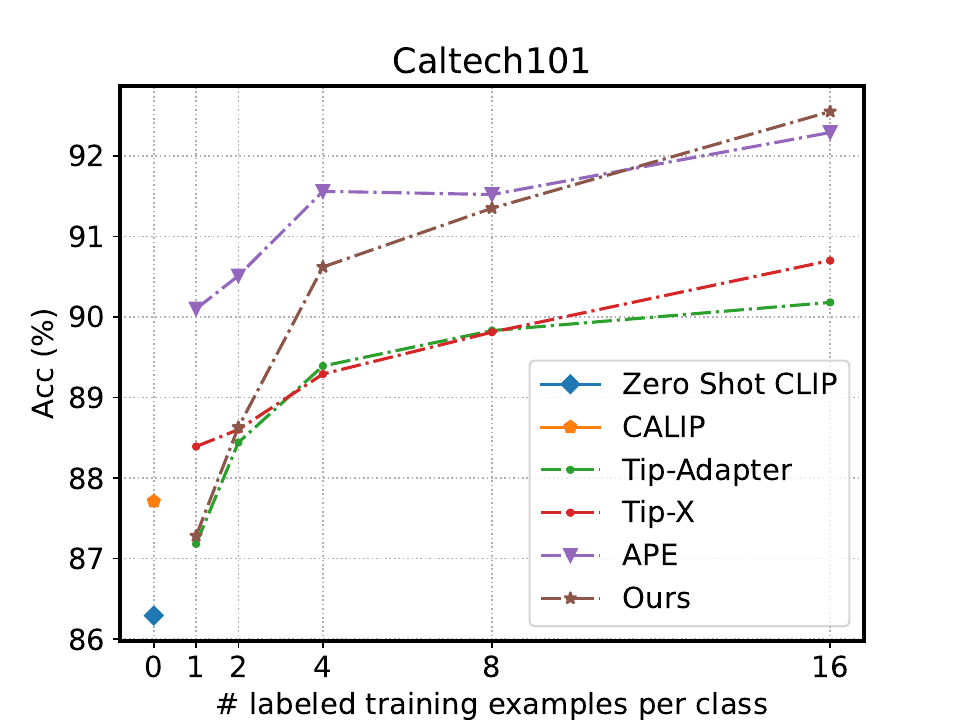}
    \end{minipage}
    \begin{minipage}{0.24\textwidth}
        \centering
        \includegraphics[width=1.13\textwidth]{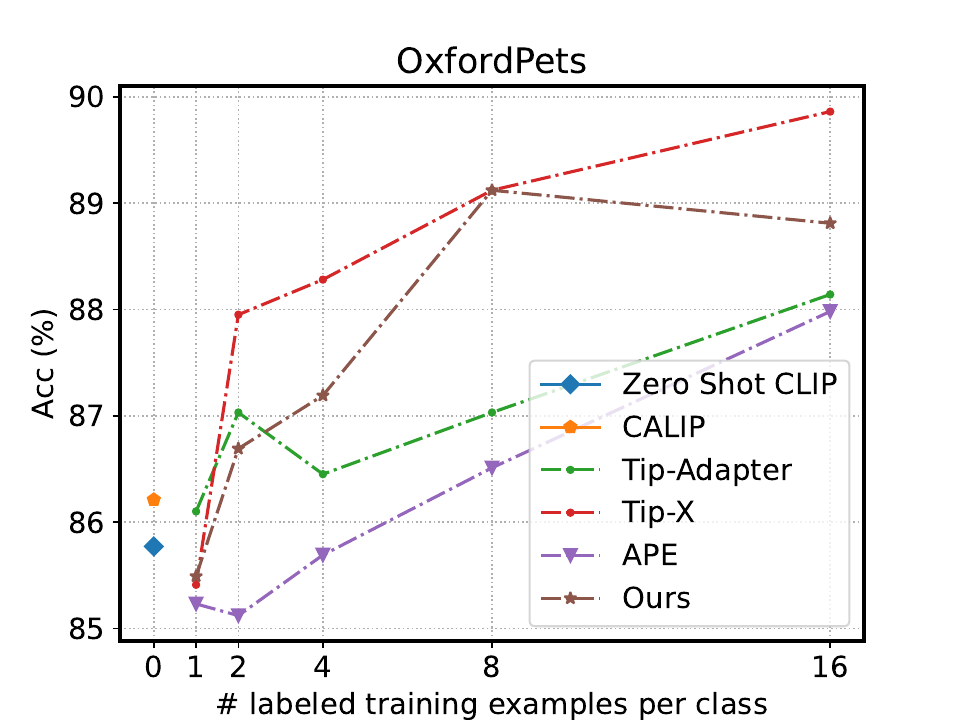}
    \end{minipage}
    \\
    \begin{minipage}{0.24\textwidth}
        \centering
        \includegraphics[width=1.13\textwidth]{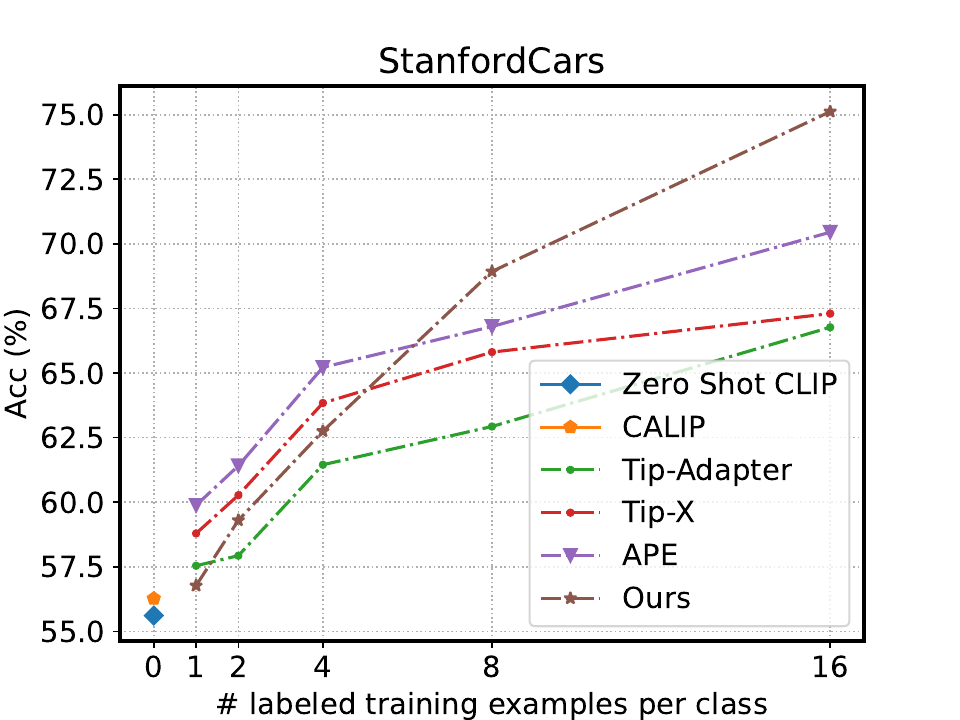}
    \end{minipage}
    \begin{minipage}{0.24\textwidth}
        \centering
        \includegraphics[width=1.13\textwidth]{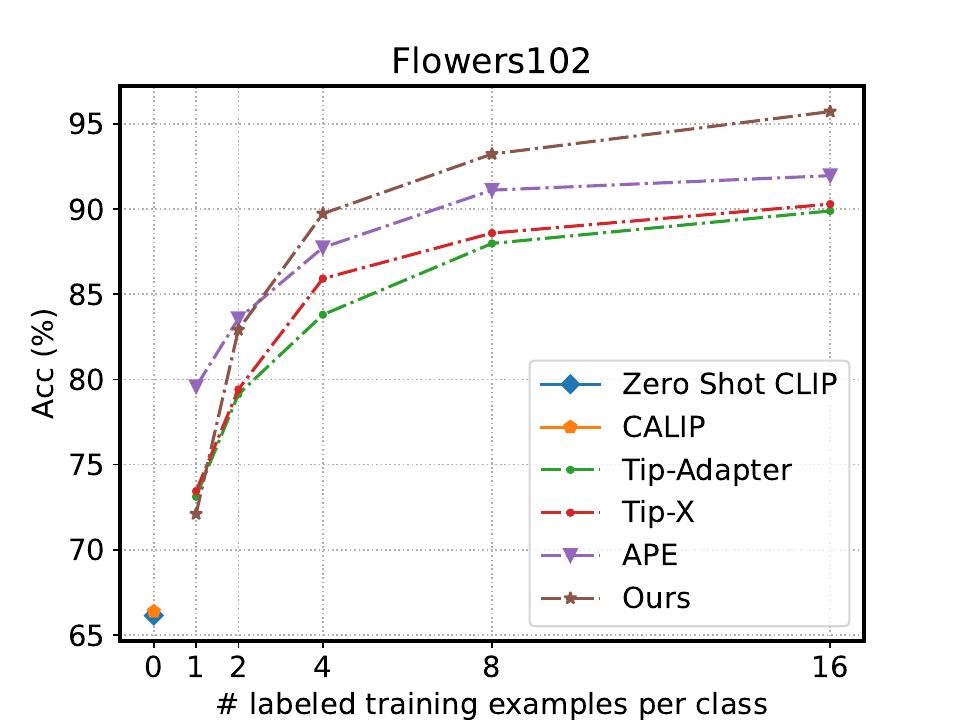}
    \end{minipage}
    \begin{minipage}{0.24\textwidth}
        \centering
        \includegraphics[width=1.13\textwidth]{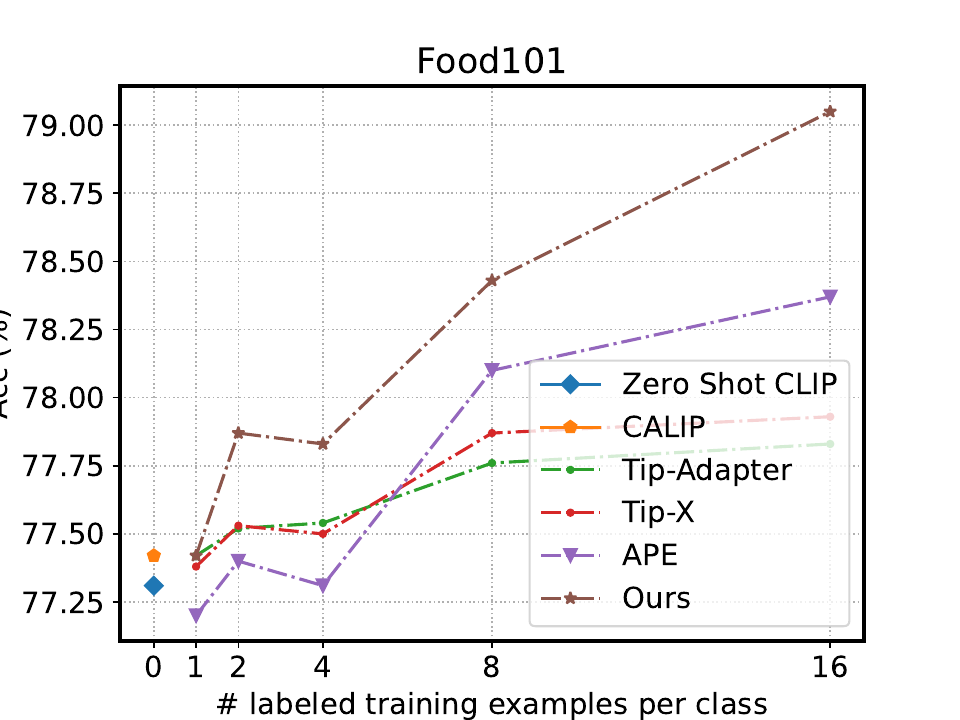}
    \end{minipage}
    \begin{minipage}{0.24\textwidth}
        \centering
        \includegraphics[width=1.13\textwidth]{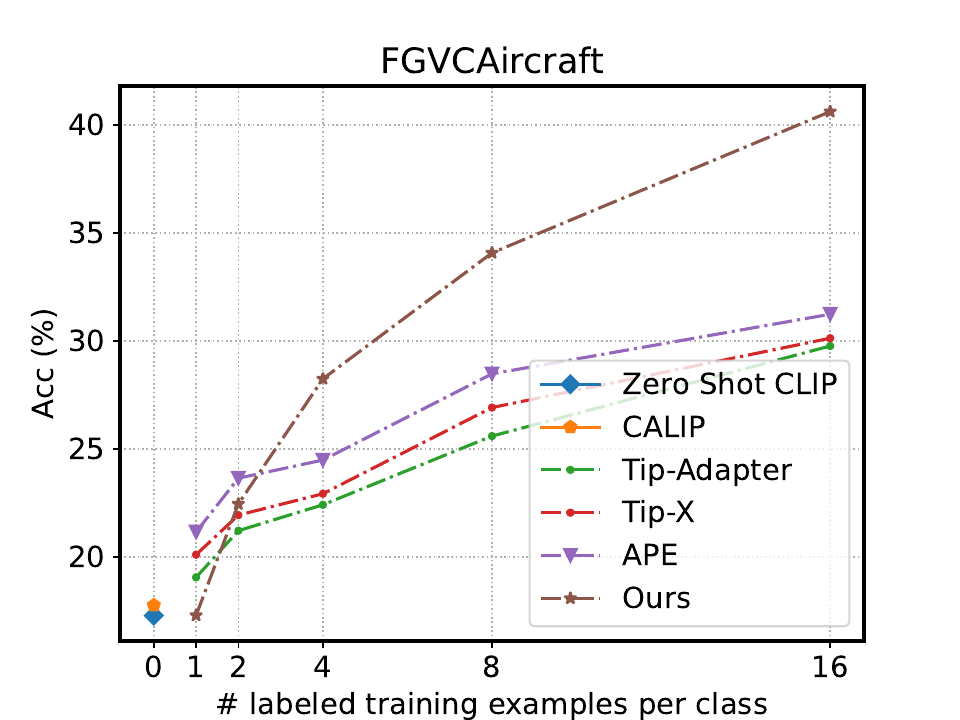}
    \end{minipage}
    \\
    \begin{minipage}{0.24\textwidth}
        \centering
        \includegraphics[width=1.13\textwidth]{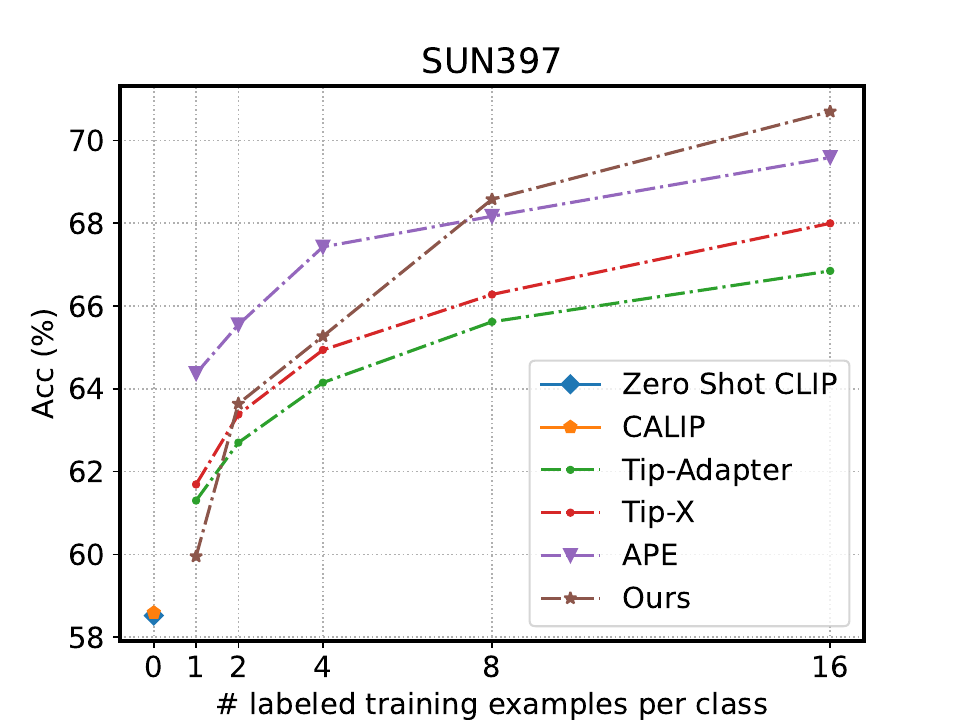}
    \end{minipage}
    \begin{minipage}{0.24\textwidth}
        \centering
        \includegraphics[width=1.13\textwidth]{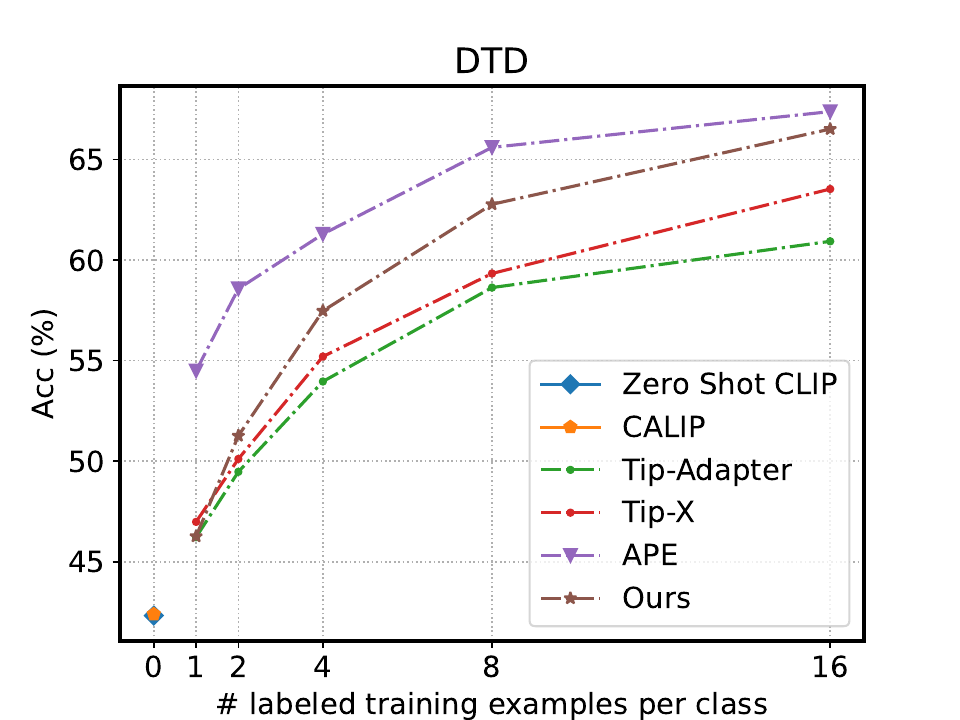}
    \end{minipage}
    \begin{minipage}{0.24\textwidth}
        \centering
        \includegraphics[width=1.13\textwidth]{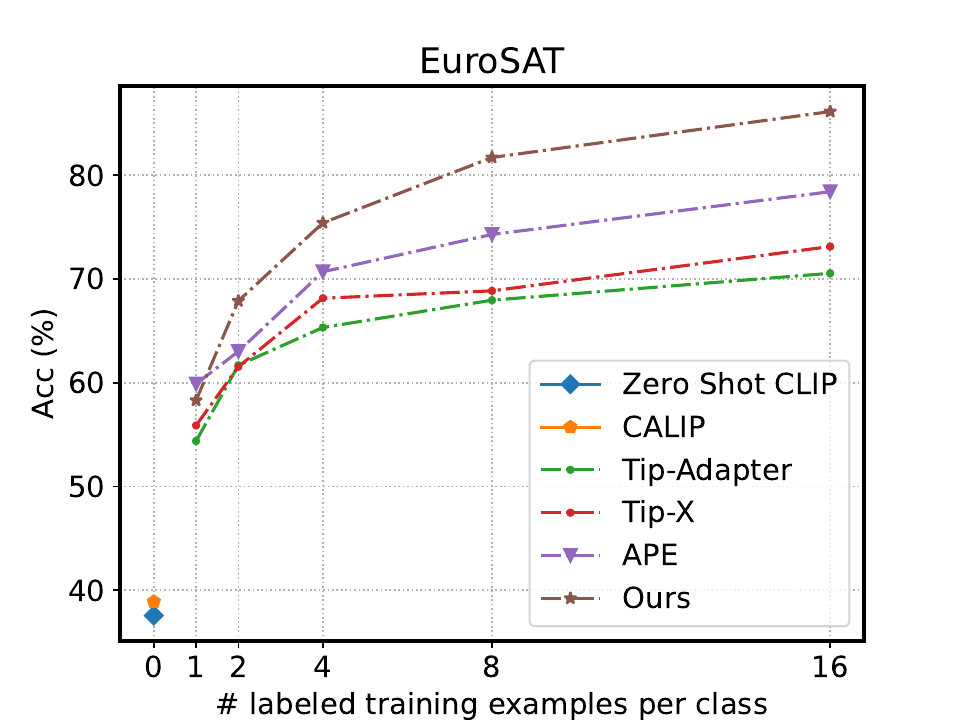}
    \end{minipage}
    \begin{minipage}{0.24\textwidth}
        \centering
        \includegraphics[width=1.13\textwidth]{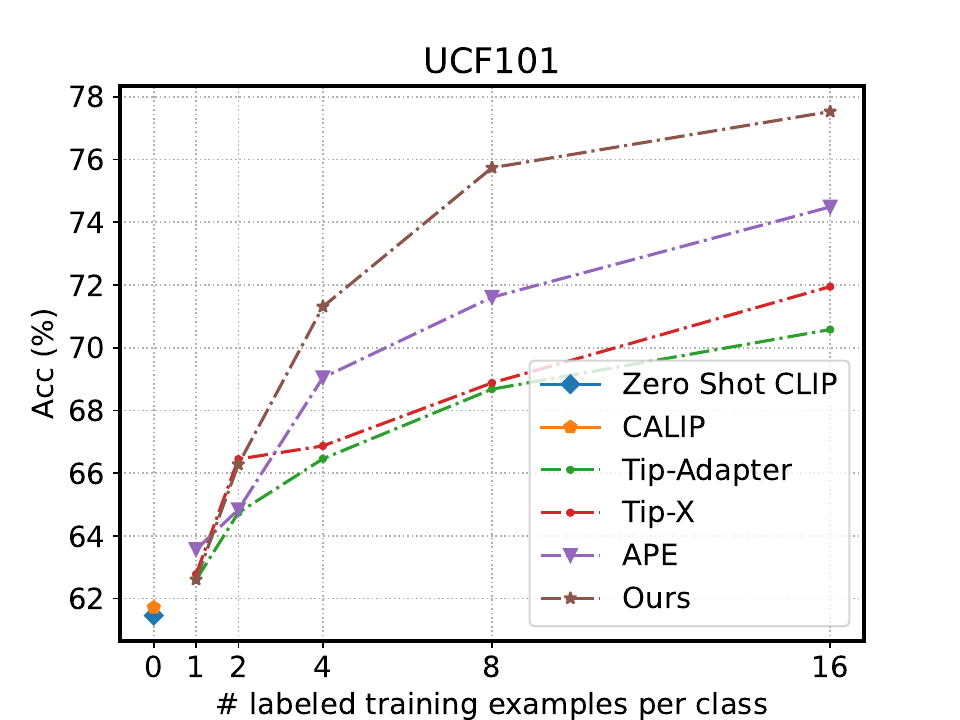}
    \end{minipage}
    \caption{
        \textbf{Results of few-shot classification on the 11 datasets.}
        We evaluate the performance of our proposed method against five training-free methods under 1, 2, 4, 8, and 16-shot settings. 
        The models are trained using ResNet-50 CLIP.
        Our method outperforms the baselines significantly.
  }
  \vspace{-10pt}
    \label{fig:fewshots}
\end{figure}

\textbf{Baselines.} We compare our method with two kinds of methods: training-required methods and training-free methods.
\textit{For training-required methods}, we consider four baselines: 
\textbf{(1) linear probe~\citep{radford2021learning}}: Following CLIP, we train a linear classifier on top of high-quality pre-trained CLIP vision encoder's features. 
\textbf{(2) CoOp~\citep{zhou2022learning}}: CoOp proposes to learn context prompts for downstream datasets through back-propagation. 
For comparison, we choose the best version of CoOp with 16 learnable prompts. 
\textbf{(3) CLIP-Adapter~\citep{gao2021clip}}: CLIP-Adapter proposes to train task-specific adapters to adjust the visual and textual representations.
\textbf{(4) Tip-Adapter-F~\citep{zhang2022tip}}: Tip-Adapter-F proposes to build a cache model using the training data to construct the adapter, which is then ensembled with the zero-shot classifier and fine-tuned during training.
\textit{For training-free methods}, we consider five baselines: 
\textbf{(1) Zero-Shot CLIP~\citep{radford2021learning}}: Following CLIP, we build the zero-shot classifier using zero-shot prompts such as ``a photo of a \{class\}".
\textbf{(2) CALIP~\citep{guo2022calip}}: CALIP builds a parameter-free attention module to boost CLIP.
\textbf{(3) Tip-Adapter~\citep{zhang2022tip}}: Tip-Adaper is the training-free version of Tip-Adapter-F. 
It builds the adapter with training data and ensembles it with the zero-shot classifier without training.
\textbf{(4) Tip-X~\citep{udandarao2022sus}}: Tip-X proposes retrieving images from LAION-5B~\citep{schuhmann2022laion} or Stable Diffusion~\citep{rombach2022high} to build the cache of Tip-Adapter.
\textbf{(5) APE~\citep{zhu2023not}}: APE adds a refinement module to Tip-Adapter, which minimizes the inter-class visual similarity and improves the text-image alignment. 
\setlength{\tabcolsep}{6pt}
\begin{table}[tbp]
  \centering
  \caption{
        \textbf{Results of few-shot classification on 11 datasets.}
        We report the performance of our method against training-free and training-required baselines on 16-shot datasets. 
        As shown in the table, our method greatly outperforms training-free baselines on average across the 11 datasets and achieves comparable performance as training-required methods.
        \textcolor{blue}{Blue} denotes the highest results of training-required methods.
        \textbf{Bold} denotes the highest results of training-free methods.
}
  \resizebox{1.0\textwidth}{!}{
    \begin{tabular}{ll|c|ccccccccccc|c}
    \toprule
    \multicolumn{2}{l|}{\textbf{Method}} & \multicolumn{1}{c|}{\textbf{Train}} & \textbf{Pets} & \textbf{Flowers} & \textbf{FGVC} & \textbf{DTD} & \textbf{EuroSAT} & \textbf{Cars} & \textbf{Food} & \textbf{SUN} & \textbf{Caltech} & \textbf{UCF} & \textbf{ImageNet} & \textbf{Average} \\
    \midrule
    \multicolumn{2}{l|}{linear-probe} & \ding{51}     & 76.42  & \color{blue}{94.95}  & \color{blue}{36.39}  & 63.97  & 82.76  & 70.08  & 70.17  & 67.15  & 90.63  & 73.72  & 55.87  & 71.10  \\
    \multicolumn{2}{l|}{CoOp} & \ding{51}     & 87.01  & 94.51  & 31.26  & 63.58  & 83.53  & 73.36  & 74.67  & 69.26  & 91.83  & 75.71  & 62.95  & 73.42  \\
    \multicolumn{2}{l|}{CLIP-Adapter} & \ding{51} & 87.84  & 93.90  & 32.10  & 65.96  & 84.43  & 74.01  & 78.25  & 69.55  & 92.49  & 76.76  & 63.59  & 74.44  \\
    \multicolumn{2}{l|}{Tip-Adapter-F} & \ding{51} & \color{blue}{89.70}  & 94.80  & 35.55  & \color{blue}{66.55}  & \color{blue}{84.54}  & \color{blue}{75.74}  & \color{blue}{79.43}  & \color{blue}{71.47}  & \color{blue}{92.86}  & \color{blue}{78.03}  & \color{blue}{65.51}  & \color{blue}{75.83}  \\
    \midrule
    \multicolumn{2}{l|}{Zero-Shot CLIP} & \ding{55} & 85.77  & 66.14  & 17.28  & 42.32  & 37.56  & 55.61  & 77.31  & 58.52  & 86.29  & 61.46  & 58.18  & 58.77  \\
    \multicolumn{2}{l|}{CALIP} & \ding{55} & 86.21  & 66.38  & 17.76  & 42.39  & 38.90  & 56.27  & 77.42  & 58.59  & 87.71  & 61.72  & 60.57  & 59.45  \\
    \multicolumn{2}{l|}{Tip-Adapter} & \ding{55}     & 88.14  & 89.89  & 29.76  & 60.93  & 70.54  & 66.77  & 77.83  & 66.85  & 90.18  & 70.58  & 62.01  & 70.32  \\
    \multicolumn{2}{l|}{Tip-X} & \ding{55}  & \textbf{89.86}  & 90.29  & 30.12  & 63.53  & 73.12  & 67.30  & 77.93  & 68.00  & 90.70  & 71.95  & 62.61  & 71.40  \\
    \multicolumn{2}{l|}{APE} & \ding{55}     & 87.98  & 91.96  & 31.23  & \textbf{67.38}  & 78.40  & 70.45  & 78.37  & 69.59  & 92.29  & 74.49  & 63.43  & 73.23  \\
    \rowcolor{gray!40}  
    \multicolumn{2}{l|}{Ours} & \ding{55}     & 88.81  & \textbf{95.72}  & \textbf{40.61}  & {66.51}  & \textbf{86.12}  & \textbf{75.12}  & \textbf{79.05}  & \textbf{70.70}  & \textbf{92.55}  & \textbf{77.53}  & \textbf{63.82}  & \textbf{76.05}  \\
    \bottomrule
    \end{tabular}%
}
    \vspace{-20pt}  
  \label{tab:fewshot_numerical}%
\end{table}%

\textbf{Results.}
Figure~\ref{fig:fewshots} illustrates the performance of our method and five other training-free baselines: Zero-Shot CLIP~\citep{radford2021learning}, CALIP~\citep{guo2022calip}, Tip-Adapter~\citep{zhang2022tip}, Tip-X~\citep{udandarao2022sus}, and APE~\citep{zhu2023not}, on the 11 downstream datasets, along with their average results.
Our method outperforms all of the baselines significantly.
Specifically, under the 16-shot setting, our method greatly exceeds Zero-Shot CLIP, CALIP, Tip-Adapter, Tip-X, and APE by 17.28\%, 16.60\%, 5.73\%, 4.65\%, and 2.82\%.  
Our method surpasses the baselines on almost all datasets except OxfordPets and DTD.

In Table~\ref{tab:fewshot_numerical}, we further present the numerical results of the training-required and training-free baselines under the 16-shot setting.
In the table, our method outperforms training-free methods on 9 of 11 datasets. 
Specifically, our method achieves a lead of over 3\% compared to the second-highest results on Flowers102, FGVCAircraft, EuroSAT, StanfordCars, and UCF101.
Moreover, our method achieves comparable performance with training-required methods.
And our method achieves great improvement on FGVCAircraft and EuroSAT.
This may be because images in these datasets are unusual, where the covariance of the data is important to describe the features.

\vspace{-5pt}
\subsection{Out-of-distribution Generalization}
\vspace{-5pt}
\begin{wrapfigure}[6]{r}{0.5\textwidth}
\vspace{-25pt}
    \centering
    \captionof{table}{
    \textbf{Out-of-distribution Generalization.} 
  }
    \vspace{-7pt}
    \resizebox{0.5\textwidth}{!}{
    \begin{tabular}{llccccccc}
    \toprule
    \multicolumn{2}{l}{\multirow{2}[4]{*}{Method}} & \multicolumn{1}{c}{\multirow{2}[4]{*}{Train}} & Source & \multicolumn{4}{c}{Target}    &  \\
\cmidrule(r){4-4} \cmidrule(r){5-9} \multicolumn{2}{c}{} &       & IN. & -V2   & -Sk & -A    & -R    & Avg. \\
    \midrule
    \multicolumn{2}{l}{CLIP} & \ding{55}     & 66.73  & 60.83  & 46.15  & 47.77  & 73.96  & 57.18  \\
    \multicolumn{2}{l}{CoOp} & \ding{51} & 71.92  & 64.18  & 46.71  & 48.41  & 74.32  & 58.41  \\
    \multicolumn{2}{l}{Tip-Adapter} & \ding{55}     & 70.50  & 63.31  & \underline{48.69}  & \textbf{50.64} & \textbf{77.70} & 60.08  \\
    \multicolumn{2}{l}{Tip-Adapter-F} & \ding{51} & \textbf{73.72} & \textbf{65.73} & 48.52  & 49.39  & \underline{77.22}  & \underline{60.21}  \\
    \rowcolor{gray!40}  
    \multicolumn{2}{l}{Ours} & \ding{55}     & \underline{72.23}  & \underline{65.04}  & \textbf{48.96} & \underline{50.51}  & 76.97  & \textbf{60.37} \\
    \bottomrule
    \end{tabular}%
  }
    \label{tab:robustness}%
\end{wrapfigure}
We further conduct experiments to assess our method on out-of-distribution generalization.
Specifically, we train our model using the 16-shot setting on ImageNet~\citep{deng2009imagenet}. 
Subsequently, we transfer the model directly to target datasets, which included ImageNetV2~\citep{recht2019imagenet}, ImageNet-Sketch~\citep{wang2019learning}, ImageNet-A~\citep{hendrycks2021natural}, and ImageNet-R~\citep{hendrycks2021many}.

As presented in Table~\ref{tab:robustness}, we choose CLIP, CoOp, Tip-Adapter, and Tip-Adapter-F for comparison.
And these methods are based on ViT-B/16-based CLIP.
Without requiring any training, our method achieves the highest results on average over the four target datasets.
These results indicate that our model is more advantageous in dealing with out-of-distribution generalization and reduces the risk of overfitting on the source dataset.
\vspace{-5pt}
\subsection{Results on Imbalanced Learning}
\vspace{-5pt}
\setlength{\tabcolsep}{6pt}
\begin{table}[htbp]
  \centering
  \caption{
    \textbf{Results of imbalanced learning on ImageNet-LT and Places-LT datasets.} 
    All models are trained on ResNet-101 CLIP.
    We compare our method with Zero-Shot CLIP~\citep{radford2021learning}, linear probe, full fine-tune, Balanced Softmax~\citep{ren2020balanced}, CRT~\citep{kang2020decoupling}, MARC~\citep{wang2023margin}, and their variants with Decoder~\citep{wang2023exploring}.
}
\vspace{-10pt}
  \resizebox{1.0\textwidth}{!}{
    \begin{tabular}{ccc|c|r|r|r|r|r|r|r|r|r|r}
    \toprule
    \multicolumn{3}{c|}{\multirow{2}[4]{*}{Method}} & \multicolumn{1}{c|}{\multirow{2}[4]{*}{Train}} & \multicolumn{5}{c|}{ImageNet-LT}      & \multicolumn{5}{c}{Places-LT} \\
\cmidrule{5-14}    \multicolumn{3}{c|}{} &       & \multicolumn{1}{c|}{Many} & \multicolumn{1}{c|}{Medium} & \multicolumn{1}{c|}{Few} & \multicolumn{1}{c|}{Overall} & \multicolumn{1}{c|}{F1} & \multicolumn{1}{c|}{Many} & \multicolumn{1}{c|}{Medium} & \multicolumn{1}{c|}{Few} & \multicolumn{1}{c|}{Overall} & \multicolumn{1}{c}{F1} \\
    \midrule
    \multicolumn{3}{l|}{Zero-Shot CLIP} & \ding{51}     & 59.57  & 53.57  & 52.81  & 53.62  & 52.50  & 36.23  & 30.43  & \textbf{37.89} & 32.17  & 30.85  \\
    \multicolumn{3}{l|}{Linear probe} & \ding{51}     & 24.23  & 0.00  & 0.00  & 9.33  & 4.97  & 23.50  & 0.20  & 0.00  & 8.52  & 4.81  \\
    \multicolumn{3}{l|}{Full finetune} & \ding{51}     & 74.49  & 52.82  & 26.66  & 57.61  & 55.86  & 47.32  & 28.66  & 11.80  & 32.08  & 30.40  \\
    \multicolumn{3}{l|}{Decoder + Softmax} & \ding{51}     & 66.93  & 42.09  & 15.32  & 48.01  & 45.21  & 28.42  & 15.80  & 10.36  & 14.95  & 13.04  \\
    \multicolumn{3}{l|}{Decoder + Balanced Softmax} & \ding{51}     & 60.60  & 52.17  & 40.53  & 53.83  & 52.57  & 20.47  & 21.79  & 21.47  & 21.53  & 18.55  \\
    \multicolumn{3}{l|}{Decoder + MARC} & \ding{51}     & 58.29  & 54.73  & 46.91  & 55.04  & 54.35  & 12.83  & 25.96  & 27.31  & 25.14  & 22.08  \\
    \multicolumn{3}{l|}{Decoder + CRT} & \ding{51}     & 66.89  & 51.98  & 23.82  & 53.89  & 51.77  & 33.14  & 14.70  & 5.01  & 12.77  & 10.78  \\
    \multicolumn{3}{l|}{Full finetune Balanced Softmax} & \ding{51}     & 69.18  & 58.25  & 43.63  & 60.47  & 59.79  & 42.42  & 38.41  & 27.93  & 37.81  & 37.21  \\
    \multicolumn{3}{l|}{Full finetune CRT} & \ding{51}     & \textbf{75.69} & 56.43  & 27.47  & 59.90  & 58.21  & \textbf{47.81} & 30.77  & 13.39  & 33.51  & 33.04  \\
    \multicolumn{3}{l|}{Full finetune MARC} & \ding{51}     & 73.73  & 58.69  & 32.91  & 60.97  & 59.61  & 46.57  & 38.09  & 17.51  & 37.13  & 35.95  \\
    \rowcolor{gray!40}  
    \multicolumn{3}{l|}{Ours} & \ding{55}     & 65.72  & \textbf{61.88} & \textbf{54.35} & \textbf{62.34} & \textbf{61.63} & 44.71  & \textbf{42.67} & 35.79  & \textbf{42.07} & \textbf{40.78} \\
    \bottomrule
    \end{tabular}%
  }
  \vspace{-10pt}
  \label{tab:longtail}%
\end{table}%
For imbalanced learning, we compare our method with several imbalanced learning baselines: Zero-Shot CLIP, linear probe, full fine-tuning, and CLIP integrated with specific imbalanced algorithms, namely: MARC~\citep{wang2023margin}, CRT~\citep{kang2020decoupling}, Balanced Softmax~\citep{ren2020balanced}, and their variants with Decoder~\citep{wang2023exploring} on ImageNet-LT~\citep{liu2019large} and Places-LT~\citep{zhou2017places} dataset.

The results are shown in Table~\ref{tab:longtail}.
We report the results in terms of overall accuracy, many-shot accuracy, medium-shot accuracy, and few-shot accuracy, as well as the F1 score.
Specifically, we achieve improvements against Zero-Shot CLIP of 8.72\%  and 9.90\% for ImageNet-LT and Places-LT datasets on overall accuracy, respectively.
It is worth noting that our method surpasses previous training-required methods, even those trained using imbalanced algorithms.
Our method primarily enhances CLIP's performance in terms of medium-shot and few-shot accuracy.
This improvement can be attributed to the identical covariance assumption in GDA, which transfers the knowledge of feature distribution from many-shot classes to medium- and few-shot classes.

\subsection{Results on Base-to-new Generalization}
\begin{wrapfigure}[7]{r}{0.35\textwidth}
\vspace{-20pt}
    \centering
    \captionof{table}{
    \textbf{Average results over 11 datasets on base-to-new generalization.} 
  }
  \vspace{-10pt}
   \resizebox{0.35\textwidth}{!}{
        \begin{tabular}{lccc|c}
        \toprule
              & train & base  & new   & \textbf{H} \\
        \midrule
        CLIP  & \ding{55} & 69.34  & 74.22  & 71.70  \\
        CoOp  & \ding{51} & 82.69  & 63.22  & 71.66  \\
        CoCoOp & \ding{51} & 80.47  & 71.69  & 75.83  \\
        KgCoOp & \ding{51} & 80.73  & 73.60  & 77.00  \\
        \rowcolor{gray!40}  
        Ours  & \ding{55} & \textbf{83.96}  & \textbf{74.53}  & \textbf{78.72}  \\
        \bottomrule
        \end{tabular}%
       }
    \label{tab:b2n}%
\end{wrapfigure}
Table~\ref{tab:b2n} presents the results on base-to-new generalization, which show that our approach outperforms the other methods in terms of base accuracy, new accuracy, and their harmonic mean. 
On average across 11 datasets, our method surpasses CLIP, CoOp, CoCoOp, and KgCoOp by 0.31\%, 11.31\%, 2.84\%, and 0.93\% in terms of new accuracy, and by 7.02\%, 7.06\%, 2.89\%, and 1.72\% in terms of the harmonic mean. 
Detailed results on each dataset can be referred to Figure~\ref{tab:base2new} in the Appendix.
\subsection{Results on Unsupervised Learning}
\setlength{\tabcolsep}{6pt}
\begin{table}[htbp]
  \centering
  \caption{
    \textbf{Results of unsupervised learning. }
    We compare our method with three baseline methods: Zero-Shot CLIP~\citep{radford2021learning}, POUF~\citep{tanwisuth2023pouf} and UPL~\citep{huang2022unsupervised}.
}
\vspace{-10pt}
  \resizebox{1.0\textwidth}{!}{
    \begin{tabular}{llcccccccccccc}
    \toprule
    \multicolumn{2}{c}{\textbf{Method}} & \textbf{Pet} & \textbf{Flo} & \textbf{FGVC} & \textbf{DTD} & \textbf{EuroSAT} & \textbf{Cars} & \textbf{Food} & \textbf{SUN} & \textbf{Cal} & \textbf{UCF} & \textbf{IN} & \textbf{Avg.} \\
    \midrule
    \multicolumn{2}{l}{CLIP} & 85.77  & 66.14  & 17.28  & 42.32  & 37.56  & 55.61  & 77.31  & 58.52  & 86.29  & 61.46  & 58.18  & 58.77  \\
    \multicolumn{2}{l}{POUF} & 88.00 & 66.71 & 16.67 & 41.49 & 42.06 & 57.43 & 74.70 & 58.61 & 86.92 & 61.05 & 55.16 & 58.98 \\
    \multicolumn{2}{l}{UPL} & 88.28  & 68.90  & 17.34  & 46.57  & \textbf{54.83} & \textbf{62.13} & 77.58  & \textbf{63.98} & \textbf{89.94} & 67.17  & 60.51  & 63.38  \\
    \rowcolor{gray!40}  
    \multicolumn{2}{l}{Ours} & \textbf{89.90} & \textbf{72.65} & \textbf{18.69} & \textbf{46.81} & 49.92  & 60.78  & \textbf{78.25} & 63.60  & 87.53  & \textbf{68.70} & \textbf{61.21} & \textbf{63.46} \\
    \bottomrule
    \end{tabular}%
}
\vspace{-10pt}
  \label{tab:unsupervised}%
\end{table}%
In unsupervised learning, the estimation of mean vectors and covariance matrices in GDA is performed by directly applying the EM algorithm for Gaussian Mixture Model (GMM).
The results are shown in Table~\ref{tab:unsupervised}.
It is noteworthy that this straightforward approach significantly enhances the performance of CLIP in downstream tasks when utilizing unlabeled data. 
Furthermore, our method consistently outperforms the Zero-Shot CLIP by an average margin of 4.69\% across all 11 datasets.
Moreover, when compared to the three baseline methods, our approach achieves the highest results on 7 out of the 11 datasets. 
These results clearly indicate the effectiveness of our method.

\begin{minipage}[!]{0.48\textwidth}
\centering
  \resizebox{1.0\textwidth}{!}{
    \begin{tabular}{cc|rrrrr}
    \toprule
    \multicolumn{2}{c|}{Method} & \multicolumn{1}{c}{1} & \multicolumn{1}{c}{2} & \multicolumn{1}{c}{4} & \multicolumn{1}{c}{8} & \multicolumn{1}{c}{16} \\
    \midrule
    \multicolumn{2}{c|}{Moore-Penrose} & 53.87  & 64.22  & 72.74  & 79.21  & 84.16  \\
    \multicolumn{2}{c|}{EM} & 54.28  & 50.85  & 56.86  & 70.69  & 74.35  \\
    \multicolumn{2}{c|}{GraphicalLasso} & 57.69  & 64.22  & 71.70  & 76.89  & 76.56  \\
    \multicolumn{2}{c|}{LedoitWolf} & \textbf{59.59} & 67.01  & 74.23  & 80.52  & 84.33  \\
    \multicolumn{2}{c|}{OAS} & 59.58  & 66.81  & 74.25  & 80.46  & 84.31  \\
    \rowcolor{gray!40}  
    \multicolumn{2}{c|}{KS} & 58.30  & \textbf{67.88} & \textbf{75.40} & \textbf{81.70} & \textbf{86.12} \\
    \bottomrule
    \end{tabular}%
  }
  \vspace{-5pt}
  \captionof{table}{
    Comparison of different precision matrix estimation methods on EuroSAT, including Moore-Penrose~\citep{penrose1955generalized}, EM~\citep{efron1976multivariate}, GraphicalLasso~\citep{friedman2008sparse}, LedoitWolf~\citep{ledoit2004well}, OAS~\citep{chen2010shrinkage}, and KS~\citep{kubokawa2008estimation}. The grey color denotes the one used in the paper.
}
  \label{tab:precision}%
\end{minipage}
\hspace{20pt}
\begin{minipage}[!]{0.4\textwidth}
\centering
    \includegraphics[trim={0.15cm 0.15cm 0.15cm 0.15cm},clip,width=1.0\textwidth]{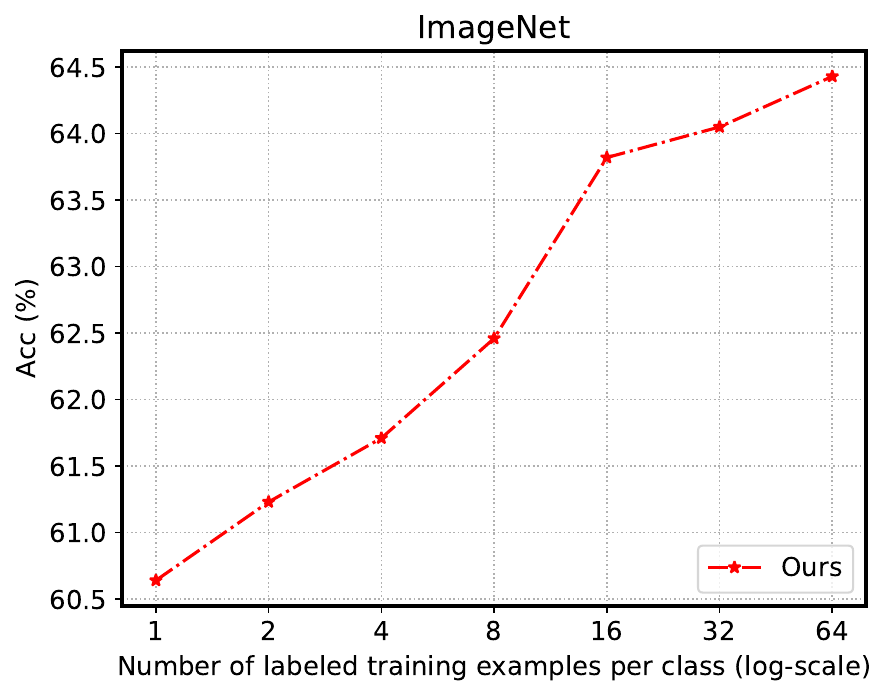}
    \captionof{figure}{
        We trained our method on ImageNet with more shots.
        The x-axis is presented on a logarithmic scale.
}
    \label{fig:moreshots}
\end{minipage}
\subsection{Ablation Study}
\vspace{-5pt}
\textbf{Effectiveness of Precision Matrix Estimation.}
The estimation of the precision matrix is challenging due to the limited data and bias problem.
To address this, we employ the empirical Bayes ridge-type estimator (KS) in the paper, which is specifically designed for scenarios where the sample size is smaller than the dimension.
We compare it with other robust precision estimation techniques, including Moore-Penrose, the estimator in Efron and Morris (EM), GraphicalLasso, LedoitWolf, and OAS.
As shown in Table~\ref{tab:precision}, the empirical Bayes ridge-type estimator achieves the best results, which shows its effectiveness.

\vspace{-5pt}
\textbf{Effectiveness of Increased Sample Size.}
We further train our method with more training data.
Figure~\ref{fig:moreshots} illustrates the results of training our model on ImageNet, using 1, 2, 4, 8, 16, 32, and 64 shots per class. 
The x-axis is presented on a logarithmic scale.
We observe that the model performance increases with an increase in the number of data, and it exhibits a linear relationship with the logarithm of the number of data.
This indicates that our approach is not restricted to few-shot learning, but instead has the ability to improve consistently with an increase in the number of samples.

\begin{minipage}[!]{0.48\textwidth}
    \centering
    \resizebox{1.0\textwidth}{!}{
    \begin{tabular}{llcccc}
    \toprule
    \multicolumn{2}{l}{Method} & Acc.(\%) & Param. & Train.Set & Train.Time \\
    \midrule
    \multicolumn{2}{l}{ResNet-50} & 74.2  & 25.6M & full set & $>1$ day \\
    \multicolumn{2}{l}{ResNet-101} & 77.4  & 44.5M & full set & $>1$ day \\
    \midrule
    \multicolumn{2}{l}{DeiT-T} & 72.2  & 6.0M  & full set & $>1$ day \\
    \multicolumn{2}{l}{DeiT-S} & 79.9  & 22.1M & full set & $>1$ day \\
    \midrule
    \multicolumn{2}{l}{Tip-Adapter} & 76.1  & 0M     & 16-shot & 0 \\
    \multicolumn{2}{l}{Tip-Adapter*} &  -   &  -  & full set  & - \\
    \midrule
    \multicolumn{2}{l}{Tip-Adapter-F} & 79.4  & 6.2M   & 16-shot & 6 min \\
    \multicolumn{2}{l}{Tip-Adapter-F*} & -     & -   & full set & - \\
    \midrule
    \rowcolor{gray!40}  
    \multicolumn{2}{l}{Ours} & 79.1  & 0M     & 16-shot & 1.6 sec \\
    \rowcolor{gray!40}  
    \multicolumn{2}{l}{Ours} & \textbf{80.0} & 0M     & full set & 3.6 sec \\
    \bottomrule
    \end{tabular}%
}
    \captionof{table}{
        * denotes that the model is out-of-memory.
        Comparison between Tip-Adapter, Tip-Adapter-F, and conventional methods, ResNet and DeiT, trained by full training set on ImageNet.
        The training time is tested on a single NVIDIA GeForce RTX 3090.
}
    \label{fig:fulltrain}
\end{minipage}
\hspace{0.5cm}
\begin{minipage}[!]{0.40\textwidth}
  \centering
  \includegraphics[trim={0.15cm 0.15cm 0.15cm 0.15cm},clip,width=1.0\textwidth]
  {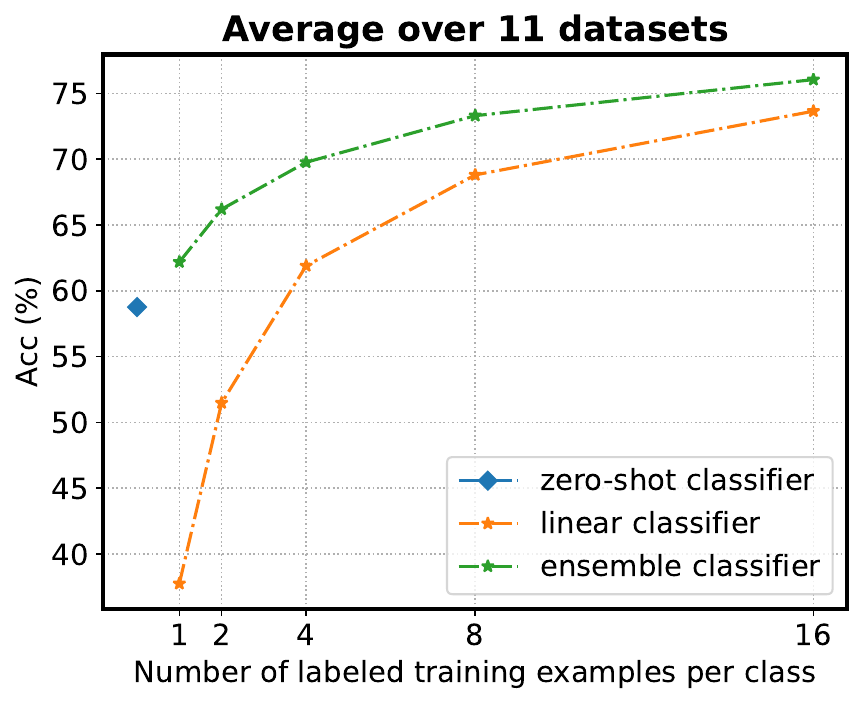}
  \vspace{-20pt}
  \captionof{figure}{
    The figure depicts the performance of the CLIP zero-shot classifier, our linear classifier, and their ensemble in the different shot settings on average across 11 datasets.
}
  \label{fig:classifier}%
\end{minipage}

\vspace{-5pt}
\textbf{Effectiveness of Ensemble Classifier.}
We evaluate the effectiveness of the ensemble of linear classifiers, as presented in Eq.~(\ref{eq:ensemble}).
Figure~\ref{fig:classifier} shows the performance of the CLIP zero-shot classifier, our linear classifier, and the ensemble classifier in few-shot classification on average on the 11 datasets.
We observe that directly using the linear classifier sometimes produces worse results than using the zero-shot classifier.
This can be attributed to inaccuracies in the estimated precision matrix, leading to a poor classifier.
However, when the classifiers are ensembled according to Eq.~(\ref{eq:ensemble}), the ensemble classifier outperforms both individual classifiers in all settings, demonstrating its effectiveness.

\vspace{-5pt}
\textbf{Comparison to Fully-trained Methods.}
In Table~\ref{fig:fulltrain}, we compare efficient fine-tuning methods, Tip-Adapter, Tip-Adapter-F, and our proposed method, with conventional fully trained methods such as ResNet~\citep{he2016deep} and DeiT~\citep{touvron2021training}.
We adopt ViT-L/14  CLIP for efficient fine-tuning methods.
Although Tip-Adapter and Tip-Adapter-F achieve comparable performance to conventional methods~\citep{he2016deep, touvron2021training}, they fail to train on full set as they need to cache all the training data, which leads to OOM error.
In contrast, our proposed method does not have this problem since we only store the classifier parameters.
Therefore, our approach can perform well not only on few-shot but also on the full training set.
Furthermore, without requiring any training, our approach achieves the highest performance compared to both efficient fine-tuning methods and conventional training methods.
\vspace{-10pt}
\section{Conclusion}
\vspace{-10pt}
In this paper, we revisit Gaussian Discriminant Analysis (GDA) with CLIP as a hard-to-beat training-free adaptation method.
Without any training, we can directly obtain the classifier from the mean vectors and covariance of the training dataset.
We conduct extensive experiments of our method on CLIP few-shot classification and imbalanced learning, and its two simple variants on base-to-new generalization and unsupervised learning.
Our method achieves state-of-the-art results against previous training-free methods and is comparable to or even better than training-required methods.
These results demonstrate the effectiveness of our method. 
In the future, we will explore the application of our method in dense prediction tasks and other scenarios such as test-time adaptation.

\section*{Reproducibility Statement}
In this paper, we provide a comprehensive overview of the datasets, training procedures, and evaluation settings, which are thoroughly discussed in Section~\ref{subsec:setup}. Detailed statistics for the datasets, prompt templates, and pseudocode can be found in Appendix~\ref{appendix:b}. To ensure the reproducibility of our method, we have also made the source code and scripts available in the supplementary materials.

\bibliography{iclr2024_conference}

\begin{thebibliography}{57}
\providecommand{\natexlab}[1]{#1}
\providecommand{\url}[1]{\texttt{#1}}
\expandafter\ifx\csname urlstyle\endcsname\relax
  \providecommand{\doi}[1]{doi: #1}\else
  \providecommand{\doi}{doi: \begingroup \urlstyle{rm}\Url}\fi

\bibitem[Bishop(2006)]{bishop2006pattern}
Christopher Bishop.
\newblock In \emph{Pattern recognition and machine learning}. Springer, 2006.

\bibitem[Bossard et~al.(2014)Bossard, Guillaumin, and
  Van~Gool]{bossard2014food}
Lukas Bossard, Matthieu Guillaumin, and Luc Van~Gool.
\newblock Food-101--mining discriminative components with random forests.
\newblock In \emph{ECCV}, 2014.

\bibitem[Chen et~al.(2023)Chen, Yao, Song, Li, Rao, and Zhang]{chen2022prompt}
Guangyi Chen, Weiran Yao, Xiangchen Song, Xinyue Li, Yongming Rao, and Kun
  Zhang.
\newblock Prompt learning with optimal transport for vision-language models.
\newblock In \emph{ICLR}, 2023.

\bibitem[Chen et~al.(2010)Chen, Wiesel, Eldar, and Hero]{chen2010shrinkage}
Yilun Chen, Ami Wiesel, Yonina~C Eldar, and Alfred~O Hero.
\newblock Shrinkage algorithms for mmse covariance estimation.
\newblock \emph{IEEE Trans. Signal Process.}, 58\penalty0 (10):\penalty0
  5016--5029, 2010.

\bibitem[Cimpoi et~al.(2014)Cimpoi, Maji, Kokkinos, Mohamed, and
  Vedaldi]{cimpoi2014describing}
Mircea Cimpoi, Subhransu Maji, Iasonas Kokkinos, Sammy Mohamed, and Andrea
  Vedaldi.
\newblock Describing textures in the wild.
\newblock In \emph{CVPR}, 2014.

\bibitem[Deng et~al.(2009)Deng, Dong, Socher, Li, Li, and
  Fei-Fei]{deng2009imagenet}
Jia Deng, Wei Dong, Richard Socher, Li-Jia Li, Kai Li, and Li~Fei-Fei.
\newblock Imagenet: A large-scale hierarchical image database.
\newblock In \emph{CVPR}, 2009.

\bibitem[Ding et~al.(2022)Ding, Xue, Xia, and Dai]{ding2022decoupling}
Jian Ding, Nan Xue, Gui-Song Xia, and Dengxin Dai.
\newblock Decoupling zero-shot semantic segmentation.
\newblock In \emph{CVPR}, 2022.

\bibitem[Efron \& Morris(1976)Efron and Morris]{efron1976multivariate}
Bradley Efron and Carl Morris.
\newblock Multivariate empirical bayes and estimation of covariance matrices.
\newblock \emph{Ann. Statist.}, pp.\  22--32, 1976.

\bibitem[Friedman et~al.(2008)Friedman, Hastie, and
  Tibshirani]{friedman2008sparse}
Jerome Friedman, Trevor Hastie, and Robert Tibshirani.
\newblock Sparse inverse covariance estimation with the graphical lasso.
\newblock \emph{Biostatistics}, 9\penalty0 (3):\penalty0 432--441, 2008.

\bibitem[Gao et~al.(2024)Gao, Geng, Zhang, Ma, Fang, Zhang, Li, and
  Qiao]{gao2021clip}
Peng Gao, Shijie Geng, Renrui Zhang, Teli Ma, Rongyao Fang, Yongfeng Zhang,
  Hongsheng Li, and Yu~Qiao.
\newblock Clip-adapter: Better vision-language models with feature adapters.
\newblock \emph{IJCV}, 132\penalty0 (2):\penalty0 581–595, 2024.

\bibitem[Gu et~al.(2022)Gu, Lin, Kuo, and Cui]{gu2021open}
Xiuye Gu, Tsung-Yi Lin, Weicheng Kuo, and Yin Cui.
\newblock Open-vocabulary object detection via vision and language knowledge
  distillation.
\newblock In \emph{ICLR}, 2022.

\bibitem[Guo et~al.(2023)Guo, Zhang, Qiu, Ma, Miao, He, and Cui]{guo2022calip}
Ziyu Guo, Renrui Zhang, Longtian Qiu, Xianzheng Ma, Xupeng Miao, Xuming He, and
  Bin Cui.
\newblock Calip: Zero-shot enhancement of clip with parameter-free attention.
\newblock In \emph{AAAI}, 2023.

\bibitem[Gupta et~al.(2022)Gupta, Narayan, Joseph, Khan, Khan, and
  Shah]{gupta2022ow}
Akshita Gupta, Sanath Narayan, KJ~Joseph, Salman Khan, Fahad~Shahbaz Khan, and
  Mubarak Shah.
\newblock Ow-detr: Open-world detection transformer.
\newblock In \emph{CVPR}, 2022.

\bibitem[He et~al.(2016)He, Zhang, Ren, and Sun]{he2016deep}
Kaiming He, Xiangyu Zhang, Shaoqing Ren, and Jian Sun.
\newblock Deep residual learning for image recognition.
\newblock In \emph{CVPR}, 2016.

\bibitem[Helber et~al.(2019)Helber, Bischke, Dengel, and
  Borth]{helber2019eurosat}
Patrick Helber, Benjamin Bischke, Andreas Dengel, and Damian Borth.
\newblock Eurosat: A novel dataset and deep learning benchmark for land use and
  land cover classification.
\newblock \emph{J-STARS}, 12\penalty0 (7):\penalty0 2217--2226, 2019.

\bibitem[Hendrycks et~al.(2021{\natexlab{a}})Hendrycks, Basart, Mu, Kadavath,
  Wang, Dorundo, Desai, Zhu, Parajuli, Guo, et~al.]{hendrycks2021many}
Dan Hendrycks, Steven Basart, Norman Mu, Saurav Kadavath, Frank Wang, Evan
  Dorundo, Rahul Desai, Tyler Zhu, Samyak Parajuli, Mike Guo, et~al.
\newblock The many faces of robustness: A critical analysis of
  out-of-distribution generalization.
\newblock In \emph{ICCV}, 2021{\natexlab{a}}.

\bibitem[Hendrycks et~al.(2021{\natexlab{b}})Hendrycks, Zhao, Basart,
  Steinhardt, and Song]{hendrycks2021natural}
Dan Hendrycks, Kevin Zhao, Steven Basart, Jacob Steinhardt, and Dawn Song.
\newblock Natural adversarial examples.
\newblock In \emph{CVPR}, 2021{\natexlab{b}}.

\bibitem[Huang et~al.(2022)Huang, Chu, and Wei]{huang2022unsupervised}
Tony Huang, Jack Chu, and Fangyun Wei.
\newblock Unsupervised prompt learning for vision-language models.
\newblock \emph{arXiv preprint arXiv:2204.03649}, 2022.

\bibitem[Jia et~al.(2021)Jia, Yang, Xia, Chen, Parekh, Pham, Le, Sung, Li, and
  Duerig]{jia2021scaling}
Chao Jia, Yinfei Yang, Ye~Xia, Yi-Ting Chen, Zarana Parekh, Hieu Pham, Quoc Le,
  Yun-Hsuan Sung, Zhen Li, and Tom Duerig.
\newblock Scaling up visual and vision-language representation learning with
  noisy text supervision.
\newblock In \emph{ICML}, 2021.

\bibitem[Joseph et~al.(2021)Joseph, Khan, Khan, and
  Balasubramanian]{joseph2021towards}
KJ~Joseph, Salman Khan, Fahad~Shahbaz Khan, and Vineeth~N Balasubramanian.
\newblock Towards open world object detection.
\newblock In \emph{CVPR}, 2021.

\bibitem[Kang et~al.(2020)Kang, Xie, Rohrbach, Yan, Gordo, Feng, and
  Kalantidis]{kang2020decoupling}
Bingyi Kang, Saining Xie, Marcus Rohrbach, Zhicheng Yan, Albert Gordo, Jiashi
  Feng, and Yannis Kalantidis.
\newblock Decoupling representation and classifier for long-tailed recognition.
\newblock In \emph{ICLR}, 2020.

\bibitem[Kim et~al.(2021)Kim, Son, and Kim]{kim2021vilt}
Wonjae Kim, Bokyung Son, and Ildoo Kim.
\newblock Vilt: Vision-and-language transformer without convolution or region
  supervision.
\newblock In \emph{ICML}, 2021.

\bibitem[Krause et~al.(2013)Krause, Stark, Deng, and Fei-Fei]{krause20133d}
Jonathan Krause, Michael Stark, Jia Deng, and Li~Fei-Fei.
\newblock 3d object representations for fine-grained categorization.
\newblock In \emph{ICCVW}, 2013.

\bibitem[Kubokawa \& Srivastava(2008)Kubokawa and
  Srivastava]{kubokawa2008estimation}
Tatsuya Kubokawa and Muni~S Srivastava.
\newblock Estimation of the precision matrix of a singular wishart distribution
  and its application in high-dimensional data.
\newblock \emph{J. Multivar. Anal.}, 99\penalty0 (9):\penalty0 1906--1928,
  2008.

\bibitem[Ledoit \& Wolf(2004)Ledoit and Wolf]{ledoit2004well}
Olivier Ledoit and Michael Wolf.
\newblock A well-conditioned estimator for large-dimensional covariance
  matrices.
\newblock \emph{J. Multivar. Anal.}, 88\penalty0 (2):\penalty0 365--411, 2004.

\bibitem[Li et~al.(2004)Li, Fergus, and Perona]{fei2004learning}
Fei-Fei Li, Rob Fergus, and Pietro Perona.
\newblock Learning generative visual models from few training examples: An
  incremental bayesian approach tested on 101 object categories.
\newblock In \emph{CVPRW}, 2004.

\bibitem[Liu et~al.(2019)Liu, Miao, Zhan, Wang, Gong, and Yu]{liu2019large}
Ziwei Liu, Zhongqi Miao, Xiaohang Zhan, Jiayun Wang, Boqing Gong, and Stella~X
  Yu.
\newblock Large-scale long-tailed recognition in an open world.
\newblock In \emph{CVPR}, 2019.

\bibitem[Lu et~al.(2019)Lu, Batra, Parikh, and Lee]{lu2019vilbert}
Jiasen Lu, Dhruv Batra, Devi Parikh, and Stefan Lee.
\newblock Vilbert: Pretraining task-agnostic visiolinguistic representations
  for vision-and-language tasks.
\newblock In \emph{NeurIPS}, 2019.

\bibitem[Lu et~al.(2022)Lu, Liu, Zhang, Liu, and Tian]{lu2022prompt}
Yuning Lu, Jianzhuang Liu, Yonggang Zhang, Yajing Liu, and Xinmei Tian.
\newblock Prompt distribution learning.
\newblock In \emph{CVPR}, 2022.

\bibitem[Maji et~al.(2013)Maji, Rahtu, Kannala, Blaschko, and
  Vedaldi]{maji2013fine}
Subhransu Maji, Esa Rahtu, Juho Kannala, Matthew Blaschko, and Andrea Vedaldi.
\newblock Fine-grained visual classification of aircraft.
\newblock \emph{arXiv preprint arXiv:1306.5151}, 2013.

\bibitem[Manli et~al.(2022)Manli, Weili, De-An, Zhiding, Tom, Anima, and
  Chaowei]{shu2022tpt}
Shu Manli, Nie Weili, Huang De-An, Yu~Zhiding, Goldstein Tom, Anandkumar Anima,
  and Xiao Chaowei.
\newblock Test-time prompt tuning for zero-shot generalization in
  vision-language models.
\newblock In \emph{NeurIPS}, 2022.

\bibitem[Mengde et~al.(2022)Mengde, Zheng, Fangyun, Yutong, Yue, Han, and
  Xiang]{xu2021}
Xu~Mengde, Zhang Zheng, Wei Fangyun, Lin Yutong, Cao Yue, Hu~Han, and Bai
  Xiang.
\newblock A simple baseline for open vocabulary semantic segmentation with
  pre-trained vision-language model.
\newblock In \emph{ECCV}, 2022.

\bibitem[Nilsback \& Zisserman(2008)Nilsback and
  Zisserman]{nilsback2008automated}
Maria-Elena Nilsback and Andrew Zisserman.
\newblock Automated flower classification over a large number of classes.
\newblock In \emph{ICVGIP}, 2008.

\bibitem[Parkhi et~al.(2012)Parkhi, Vedaldi, Zisserman, and
  Jawahar]{parkhi2012cats}
Omkar~M Parkhi, Andrea Vedaldi, Andrew Zisserman, and CV~Jawahar.
\newblock Cats and dogs.
\newblock In \emph{CVPR}, 2012.

\bibitem[Penrose(1955)]{penrose1955generalized}
Roger Penrose.
\newblock A generalized inverse for matrices.
\newblock In \emph{Math. Proc. Cambridge Philos. Soc.}, 1955.

\bibitem[Radford et~al.(2021)Radford, Kim, Hallacy, Ramesh, Goh, Agarwal,
  Sastry, Askell, Mishkin, Clark, et~al.]{radford2021learning}
Alec Radford, Jong~Wook Kim, Chris Hallacy, Aditya Ramesh, Gabriel Goh,
  Sandhini Agarwal, Girish Sastry, Amanda Askell, Pamela Mishkin, Jack Clark,
  et~al.
\newblock Learning transferable visual models from natural language
  supervision.
\newblock In \emph{ICML}, 2021.

\bibitem[Recht et~al.(2019)Recht, Roelofs, Schmidt, and
  Shankar]{recht2019imagenet}
Benjamin Recht, Rebecca Roelofs, Ludwig Schmidt, and Vaishaal Shankar.
\newblock Do imagenet classifiers generalize to imagenet?
\newblock In \emph{ICML}, 2019.

\bibitem[Ren et~al.(2020)Ren, Yu, Ma, Zhao, Yi, et~al.]{ren2020balanced}
Jiawei Ren, Cunjun Yu, Xiao Ma, Haiyu Zhao, Shuai Yi, et~al.
\newblock Balanced meta-softmax for long-tailed visual recognition.
\newblock In \emph{NeurIPS}, 2020.

\bibitem[Rombach et~al.(2022)Rombach, Blattmann, Lorenz, Esser, and
  Ommer]{rombach2022high}
Robin Rombach, Andreas Blattmann, Dominik Lorenz, Patrick Esser, and Bj{\"o}rn
  Ommer.
\newblock High-resolution image synthesis with latent diffusion models.
\newblock In \emph{CVPR}, 2022.

\bibitem[Schuhmann et~al.(2022)Schuhmann, Beaumont, Vencu, Gordon, Wightman,
  Cherti, Coombes, Katta, Mullis, Wortsman, et~al.]{schuhmann2022laion}
Christoph Schuhmann, Romain Beaumont, Richard Vencu, Cade Gordon, Ross
  Wightman, Mehdi Cherti, Theo Coombes, Aarush Katta, Clayton Mullis, Mitchell
  Wortsman, et~al.
\newblock Laion-5b: An open large-scale dataset for training next generation
  image-text models.
\newblock In \emph{NeurIPS}, 2022.

\bibitem[Soomro et~al.(2012)Soomro, Zamir, and Shah]{soomro2012ucf101}
Khurram Soomro, Amir~Roshan Zamir, and Mubarak Shah.
\newblock Ucf101: A dataset of 101 human actions classes from videos in the
  wild.
\newblock \emph{arXiv preprint arXiv:1212.0402}, 2012.

\bibitem[Su et~al.(2020)Su, Zhu, Cao, Li, Lu, Wei, and Dai]{su2019vl}
Weijie Su, Xizhou Zhu, Yue Cao, Bin Li, Lewei Lu, Furu Wei, and Jifeng Dai.
\newblock Vl-bert: Pre-training of generic visual-linguistic representations.
\newblock In \emph{ICLR}, 2020.

\bibitem[Tanwisuth et~al.(2023)Tanwisuth, Zhang, Zheng, He, and
  Zhou]{tanwisuth2023pouf}
Korawat Tanwisuth, Shujian Zhang, Huangjie Zheng, Pengcheng He, and Mingyuan
  Zhou.
\newblock Pouf: Prompt-oriented unsupervised fine-tuning for large pre-trained
  models.
\newblock In \emph{ICML}, 2023.

\bibitem[Touvron et~al.(2021)Touvron, Cord, Douze, Massa, Sablayrolles, and
  J{\'e}gou]{touvron2021training}
Hugo Touvron, Matthieu Cord, Matthijs Douze, Francisco Massa, Alexandre
  Sablayrolles, and Herv{\'e} J{\'e}gou.
\newblock Training data-efficient image transformers \& distillation through
  attention.
\newblock In \emph{ICML}, 2021.

\bibitem[Udandarao et~al.(2023)Udandarao, Gupta, and Albanie]{udandarao2022sus}
Vishaal Udandarao, Ankush Gupta, and Samuel Albanie.
\newblock Sus-x: Training-free name-only transfer of vision-language models.
\newblock In \emph{ICCV}, 2023.

\bibitem[Wang et~al.(2019)Wang, Ge, Lipton, and Xing]{wang2019learning}
Haohan Wang, Songwei Ge, Zachary Lipton, and Eric~P Xing.
\newblock Learning robust global representations by penalizing local predictive
  power.
\newblock In \emph{NeurIPS}, 2019.

\bibitem[Wang et~al.(2023{\natexlab{a}})Wang, Zhang, Hou, Wu, Wang, and
  Shinozaki]{wang2023margin}
Yidong Wang, Bowen Zhang, Wenxin Hou, Zhen Wu, Jindong Wang, and Takahiro
  Shinozaki.
\newblock Margin calibration for long-tailed visual recognition.
\newblock In \emph{ACML}, 2023{\natexlab{a}}.

\bibitem[Wang et~al.(2024)Wang, Yu, Wang, Heng, Chen, Ye, Xie, Xie, and
  Zhang]{wang2023exploring}
Yidong Wang, Zhuohao Yu, Jindong Wang, Qiang Heng, Hao Chen, Wei Ye, Rui Xie,
  Xing Xie, and Shikun Zhang.
\newblock Exploring vision-language models for imbalanced learning.
\newblock \emph{IJCV}, 132\penalty0 (1):\penalty0 224--237, 2024.

\bibitem[Wang et~al.(2023{\natexlab{b}})Wang, Liang, He, Xu, Wang, and
  Tan]{wang2023improving}
Zhengbo Wang, Jian Liang, Ran He, Nan Xu, Zilei Wang, and Tieniu Tan.
\newblock Improving zero-shot generalization for clip with synthesized prompts.
\newblock In \emph{ICCV}, 2023{\natexlab{b}}.

\bibitem[Xiao et~al.(2010)Xiao, Hays, Ehinger, Oliva, and
  Torralba]{xiao2010sun}
Jianxiong Xiao, James Hays, Krista~A Ehinger, Aude Oliva, and Antonio Torralba.
\newblock Sun database: Large-scale scene recognition from abbey to zoo.
\newblock In \emph{CVPR}, 2010.

\bibitem[Yang et~al.(2021)Yang, Liu, and Xu]{yangfree}
Shuo Yang, Lu~Liu, and Min Xu.
\newblock Free lunch for few-shot learning: Distribution calibration.
\newblock In \emph{ICLR}, 2021.

\bibitem[Yao et~al.(2023)Yao, Zhang, and Xu]{yao2023visual}
Hantao Yao, Rui Zhang, and Changsheng Xu.
\newblock Visual-language prompt tuning with knowledge-guided context
  optimization.
\newblock In \emph{CVPR}, 2023.

\bibitem[Zhang et~al.(2022)Zhang, Zhang, Fang, Gao, Li, Dai, Qiao, and
  Li]{zhang2022tip}
Renrui Zhang, Wei Zhang, Rongyao Fang, Peng Gao, Kunchang Li, Jifeng Dai,
  Yu~Qiao, and Hongsheng Li.
\newblock Tip-adapter: Training-free adaption of clip for few-shot
  classification.
\newblock In \emph{ECCV}, 2022.

\bibitem[Zhou et~al.(2017)Zhou, Lapedriza, Khosla, Oliva, and
  Torralba]{zhou2017places}
Bolei Zhou, Agata Lapedriza, Aditya Khosla, Aude Oliva, and Antonio Torralba.
\newblock Places: A 10 million image database for scene recognition.
\newblock \emph{TPAMI}, 40\penalty0 (6):\penalty0 1452--1464, 2017.

\bibitem[Zhou et~al.(2022{\natexlab{a}})Zhou, Yang, Loy, and
  Liu]{zhou2022conditional}
Kaiyang Zhou, Jingkang Yang, Chen~Change Loy, and Ziwei Liu.
\newblock Conditional prompt learning for vision-language models.
\newblock In \emph{CVPR}, 2022{\natexlab{a}}.

\bibitem[Zhou et~al.(2022{\natexlab{b}})Zhou, Yang, Loy, and
  Liu]{zhou2022learning}
Kaiyang Zhou, Jingkang Yang, Chen~Change Loy, and Ziwei Liu.
\newblock Learning to prompt for vision-language models.
\newblock \emph{IJCV}, 130\penalty0 (9):\penalty0 2337--2348,
  2022{\natexlab{b}}.

\bibitem[Zhu et~al.(2023)Zhu, Zhang, He, Zhou, Wang, Zhao, and Gao]{zhu2023not}
Xiangyang Zhu, Renrui Zhang, Bowei He, Aojun Zhou, Dong Wang, Bin Zhao, and
  Peng Gao.
\newblock Not all features matter: Enhancing few-shot clip with adaptive prior
  refinement.
\newblock \emph{ICCV}, 2023.

\end{thebibliography}
\bibliographystyle{iclr2024_conference}

\newpage
\appendix
\onecolumn

\noindent{\large Supplementary Materials Organization:}

\noindent\DoToC

\section{Details of the Method}

\subsection{Computation of Equation~(2)}
\begin{theorem}
Assuming that the features of different classes follow the Gaussian distribution with identical covariance, i.e., $(X|Y=i)\sim \mathcal{N}(\mu_i, \Sigma)$ for $i=1,2,..,K$. Then, the classification probability can be expressed as follows:
\begin{equation}
    p(y=i|x) = \frac{\exp(\mu_i^T\Sigma^{-1}x - \frac{1}{2}\mu_i^T\Sigma^{-1}\mu_i + \log p_i)}{\sum_{j=1}^K\exp(\mu_j^T\Sigma^{-1}x - \frac{1}{2}\mu_j^T\Sigma^{-1}\mu_j + \log p_j)},
\end{equation}
\end{theorem}
\begin{proof}
Since $(X|Y=i)\sim \mathcal{N}(\mu_i, \Sigma)$ for $i=1,2,..,K$, the probability of class $i$ is:
\begin{equation}
\label{eq:multi_gaussian}
    p(x|y=i) = \frac{1}{(2\pi)^\frac{d}{2}|\Sigma|^{\frac{1}{2}}}\exp(-\frac{(x - \mu_i)^T\Sigma^{-1}(x - \mu_i)}{2}),
\end{equation}
where $d$ is the feature dimension.
Later, the classification probability can be derived by using the Bayesian formula,
\begin{equation}
\begin{aligned}
    p(y=i|x) 
    & = \frac{p(x|y=i)p(y=i)}{\sum_{j=1}^{K}p(x|y=j)p(y=j)} \quad \text{(Bayesian formula)} \\
    & = \frac{\frac{1}{(2\pi)^\frac{d}{2}|\Sigma|^{\frac{1}{2}}}\exp(-\frac{(x - \mu_i)^T\Sigma^{-1}(x - \mu_i)}{2})p(y=i)}{\sum_{j=1}^{K}\frac{1}{(2\pi)^\frac{d}{2}|\Sigma|^{\frac{1}{2}}}\exp(-\frac{(x - \mu_j)^T\Sigma^{-1}(x - \mu_j)}{2})p(y=j)} \quad \text{(Using Equation~\ref{eq:multi_gaussian})} \\
    & = 
    \frac{
    {\color{red}\cancel{\frac{1}{(2\pi)^\frac{d}{2}|\Sigma|^{\frac{1}{2}}}}}
    \exp
    ({\color{red}{\cancel{-\frac{x^T\Sigma^{-1}x}{2}}}} + \mu_i\Sigma^{-1}x - \frac{\mu_i^T\Sigma^{-1}\mu_i}{2})p(y=i)}
    {\sum_{j=1}^{K}{\color{red}{\cancel{\frac{1}{(2\pi)^\frac{d}{2}|\Sigma|^{\frac{1}{2}}}}}}\exp({\color{red}{\cancel{-\frac{x^T\Sigma^{-1}x}{2}}}} + \mu_j\Sigma^{-1}x - \frac{\mu_j^T\Sigma^{-1}\mu_j}{2})p(y=j)}  \\
    & = \frac{\exp(\mu_i^T\Sigma^{-1}x - \frac{1}{2}\mu_i^T\Sigma^{-1}\mu_i + \log p_i)}{\sum_{j=1}^K\exp(\mu_j^T\Sigma^{-1}x - \frac{1}{2}\mu_j^T\Sigma^{-1}\mu_j + \log p_j)} \quad \text{(denoted $p_i = p(y=i)$)}
\end{aligned}
\end{equation}

\end{proof}

\subsection{Pseudocode}

\begin{algorithm}[H]
\caption{Pytorch-like pseudocode for our method.}
\begin{minted}[
% frame=lines,
framesep=2mm,
baselinestretch=1.2,
fontsize=\footnotesize,
linenos
]{python3}
# Input:
# - X: (N, D) visual features from CLIP visual encoder.
# - Y: (N, ) ground-truth label for the features.
# - X_test: (M, D) test visual features from CLIP visual encoder.
# - Y_test: (M, ) ground-truth label for test features.
# - W_c: (K, D) zero-shot classifier generated by prompting.
# Output:
# - acc: test accuracy.

def hard_to_beat(X, Y, X_test, Y_test, W_c):
    # 1. Compute mean vectors for each class.
    mus = []
    for i in range(K):
        idx = torch.where(Y == i)
        mus.append(X[idx].mean(dim=0))
    mus = torch.cat(mus)
    
    # 2. Estimate the precision matrix using Equation (4).
    # centered features
    centered_X = torch.cat([(X[torch.where(Y == i)] - mus[i]) for i in range(K)])
    cov = torch.cov(centered_X)
    # compute the precision matrix (inverse covariance)
    inv_cov = D * torch.inv((N - 1) * cov + trace(cov) * eye(D))
    
    # 3. Compute weight and bias using Equation (3).
    W = mus @ inv_cov
    b = log(1 / K) - 0.5 * einsum('nd, dc, nc -> n', mus, inv_cov, mus)
    
    # 4. Search the hyperparameter using the validation set.
    alpha = search_hyperparam(W_c, W, b)
    
    # 5. Test.
    test_logits = X_test @ W_c.T + alpha * (X_test @ W.T + b)
    acc = compute_acc(test_logits, Y_test)
return acc
\end{minted}
\end{algorithm}

\section{More Experimental Analysis}
\label{appendix:a}

\subsection{Base-to-new Generalization}

\textbf{Results.}
Our method can be extended to the base-to-new generalization scenario by incorporating the KNN algorithm.
To accomplish this, we utilize the text embeddings of the new classes to query the training set and select the k nearest neighbors as the training data for the new class.
Subsequently, we apply our proposed method to generate the classifier for the new classes using the synthesized dataset.
In order to compare our approach, we select CLIP~\citep{radford2021learning}, CoOp~\citep{zhou2022learning}, CoCoOp~\citep{zhou2022conditional}, and KgCoOp~\citep{yao2023visual}.

Table~\ref{tab:base2new} presents the results, which demonstrate that our approach outperforms the other methods in terms of base accuracy, new accuracy, and their harmonic mean. 
On average across 11 datasets, our method surpasses CLIP, CoOp, CoCoOp, And KgCoOp by 14.62\%, 1.27\%, 3.49\%, and 3.23\% in terms of base accuracy.
It also outperforms them by 0.31\%, 11.31\%, 2.84\%, and 0.93\% in terms of new accuracy, and by 7.02\%, 7.06\%, 2.89\%, and 1.72\% in terms of the harmonic mean. 
Moreover, our approach achieves the highest harmonic mean in 6 out of 11 datasets. These results clearly indicate the effectiveness of our approach in generalizing to new classes.

\begin{table}[!htbp]
    \begin{minipage}{0.325\textwidth}
        \centering
        \caption*{\small{(a) \textbf{Average over 11 datasets}}}
        \resizebox{1.0\textwidth}{!}{
        \vspace{-10pt}
        \begin{tabular}{lcc|c}
        \toprule
              & base  & new   & \textbf{H} \\
        \midrule
        CLIP  & 69.34  & 74.22  & 71.70  \\
        CoOp  & 82.69  & 63.22  & 71.66  \\
        CoCoOp & 80.47  & 71.69  & 75.83  \\
        KgCoOp & 80.73  & 73.60  & 77.00  \\
        \rowcolor{gray!40}  
        Ours  & \textbf{83.96}  & \textbf{74.53}  & \textbf{78.72}  \\
        \bottomrule
        \end{tabular}%
  }
    \end{minipage}
    \begin{minipage}{0.325\textwidth}
        \centering
        \caption*{\small{(b) ImageNet}}
        \resizebox{1.0\textwidth}{!}{
        \begin{tabular}{lcc|c}
        \toprule
              & base  & new   & \textbf{H} \\
        \midrule
        CLIP  & 72.43  & 68.14  & 70.22  \\
        CoOp  & \textbf{76.47}  & 67.88  & 71.92  \\
        CoCoOp & 75.98  & \textbf{70.43}  & \textbf{73.10}  \\
        KgCoOp & 75.83  & 69.96  & 72.78  \\
        \rowcolor{gray!40}  
        Ours  & 75.95  & 69.79  & 72.74  \\
        \bottomrule
        \end{tabular}%
  }
    \end{minipage}
    \begin{minipage}{0.325\textwidth}
        \centering
        \caption*{\small{(c) Caltech101}}
        \resizebox{1.0\textwidth}{!}{
        \begin{tabular}{lcc|c}
        \toprule
              & base  & new   & \textbf{H} \\
        \midrule
        CLIP  & 96.84  & 94.00  & 95.40  \\
        CoOp  & 98.00  & 89.81  & 93.73  \\
        CoCoOp & 97.96  & 93.81  & 95.84  \\
        KgCoOp & 97.72  & 94.39  & 96.03  \\
        \rowcolor{gray!40}  
        Ours  & \textbf{98.04}  & \textbf{94.51}  & \textbf{96.24}  \\
        \bottomrule
        \end{tabular}%
  }
    \end{minipage}
    
    \begin{minipage}{0.325\textwidth}
        \centering
        \caption*{\small{(d) OxfordPets}}
        \resizebox{1.0\textwidth}{!}{
        \begin{tabular}{lcc|c}
        \toprule
              & base  & new   & \textbf{H} \\
        \midrule
        CLIP  & 91.17  & 97.26  & 94.12  \\
        CoOp  & 93.67  & 95.29  & 94.47  \\
        CoCoOp & \textbf{95.20}  & {97.69}  & 
        \textbf{96.43}  \\
        KgCoOp & 94.65  & \textbf{97.76}  & 96.18  \\
        \rowcolor{gray!40}  
        Ours  & 94.10  & 97.15  & 95.60  \\
        \bottomrule
        \end{tabular}%
  }
    \end{minipage}
    \begin{minipage}{0.325\textwidth}
        \centering
        \caption*{\small{(e) StanfordCars}}
        \resizebox{1.0\textwidth}{!}{
        \begin{tabular}{lcc|c}
        \toprule
              & base  & new   & \textbf{H} \\
        \midrule
        CLIP  & 63.37  & {74.89}  & 68.65  \\
        CoOp  & 78.12  & 60.40  & 68.13  \\
        CoCoOp & 70.49  & 73.59  & 72.01  \\
        KgCoOp & 71.76  & \textbf{75.04}  & \textbf{73.36 } \\
        \rowcolor{gray!40}  
        Ours  & \textbf{78.71}  & 66.92  & {72.34}  \\
        \bottomrule
        \end{tabular}%
  }
    \end{minipage}
    \begin{minipage}{0.325\textwidth}
        \centering
        \caption*{\small{(f) Flowers102}}
        \resizebox{1.0\textwidth}{!}{
        \begin{tabular}{lcc|c}
        \toprule
              & base  & new   & \textbf{H} \\
        \midrule
        CLIP  & 72.08  & \textbf{77.80}  & 74.83  \\
        CoOp  & 97.60  & 59.67  & 74.06  \\
        CoCoOp & 94.87  & 71.75  & 81.71  \\
        KgCoOp & 95.00  & 74.73  & \textbf{83.65}  \\
        \rowcolor{gray!40}  
        Ours  & \textbf{97.78}  & 72.46  & {83.24}  \\
        \bottomrule
        \end{tabular}%
  }
    \end{minipage}

    \begin{minipage}{0.325\textwidth}
        \centering
        \caption*{\small{(g) Food101}}
        \resizebox{1.0\textwidth}{!}{
        \begin{tabular}{lcc|c}
        \toprule
              & base  & new   & \textbf{H} \\
        \midrule
        CLIP  & 90.10  & 91.22  & 90.66  \\
        CoOp  & 88.33  & 82.26  & 85.19  \\
        CoCoOp & \textbf{90.70}  & {91.29}  & {90.99}  \\
        KgCoOp & 90.50  & \textbf{91.70}  & \textbf{91.09}  \\
        \rowcolor{gray!40}  
        Ours  & 90.63  & 91.21  & 90.92  \\
        \bottomrule
        \end{tabular}%
  }
    \end{minipage}
    \begin{minipage}{0.325\textwidth}
        \centering
        \caption*{\small{(h) FGVCAircraft}}
        \resizebox{1.0\textwidth}{!}{
        \begin{tabular}{lcc|c}
        \toprule
              & base  & new   & \textbf{H} \\
        \midrule
        CLIP  & 27.19  & \textbf{36.29}  & 31.09  \\
        CoOp  & 40.44  & 22.30  & 28.75  \\
        CoCoOp & 33.41  & 23.71  & 27.74  \\
        KgCoOp & 36.21  & 33.55  & 34.83  \\
        \rowcolor{gray!40}  
        Ours  & \textbf{45.88}  & 34.09  & \textbf{39.12}  \\
        \bottomrule
        \end{tabular}%
  }
    \end{minipage}
    \begin{minipage}{0.325\textwidth}
        \centering
        \caption*{\small{(i) SUN397}}
        \resizebox{1.0\textwidth}{!}{
        \begin{tabular}{lcc|c}
        \toprule
              & base  & new   & \textbf{H} \\
        \midrule
        CLIP  & 69.36  & 75.35  & 72.23  \\
        CoOp  & 80.60  & 65.89  & 72.51  \\
        CoCoOp & 79.74  & \textbf{76.86} & 78.27  \\
        KgCoOp & 80.29  & 76.53  & 78.36  \\
        \rowcolor{gray!40}  
        Ours  & \textbf{81.95}  & 75.62  & \textbf{78.65}  \\
        \bottomrule
        \end{tabular}%
  }
    \end{minipage}

    \begin{minipage}{0.325\textwidth}
        \centering
        \caption*{\small{(j) DTD}}
        \resizebox{1.0\textwidth}{!}{
        \begin{tabular}{lcc|c}
        \toprule
              & base  & new   & \textbf{H} \\
        \midrule
        CLIP  & 53.24  & \textbf{59.90}  & 56.37  \\
        CoOp  & 79.44  & 41.18  & 54.24  \\
        CoCoOp & 77.01  & 56.00 & 64.85  \\
        KgCoOp & 77.55  & 54.99  & 64.35  \\
        \rowcolor{gray!40}  
        Ours  & \textbf{80.63}  & 59.82  & \textbf{68.69}  \\
        \bottomrule
        \end{tabular}%
  }
    \end{minipage}
    \begin{minipage}{0.325\textwidth}
        \centering
        \caption*{\small{(k) EuroSAT}}
        \resizebox{1.0\textwidth}{!}{
        \begin{tabular}{lcc|c}
        \toprule
              & base  & new   & \textbf{H} \\
        \midrule
        CLIP  & 56.48  & 64.05  & 60.03  \\
        CoOp  & 92.19  & 54.74  & 68.69  \\
        CoCoOp & 87.49  & 60.04 & 71.21  \\
        KgCoOp & 85.64  & 64.34  & 73.48  \\
        \rowcolor{gray!40}  
        Ours  & \textbf{93.28}  & \textbf{79.21}  & \textbf{85.67}  \\
        \bottomrule
        \end{tabular}%
  }
    \end{minipage}
    \begin{minipage}{0.325\textwidth}
        \centering
        \caption*{\small{(l) UCF101}}
        \resizebox{1.0\textwidth}{!}{
        \begin{tabular}{lcc|c}
        \toprule
              & base  & new   & \textbf{H} \\
        \midrule
        CLIP  & 70.53  & 77.50  & 73.85  \\
        CoOp  & 84.69  & 56.05  & 67.46  \\
        CoCoOp & 82.33  & 73.45 & 77.64  \\
        KgCoOp & 82.89  & 76.67  & 79.65  \\
        \rowcolor{gray!40}  
        Ours  & \textbf{86.63}  & \textbf{79.09}  & \textbf{82.69}  \\
        \bottomrule
        \end{tabular}%
  }
    \end{minipage}
    
    \caption{
        \textbf{Base-to-new generalization.}
        Comparison of CLIP, CoOp, CoCoOp, KgCoOp, and our method. 
        CoOp, CoCoOp, and KgCoOp are training-required methods, while our method is a training-free method.
        base and new denotes the average accuracy of base and new classes, and H denotes their harmonic mean.
}
    \label{tab:base2new}
\end{table}

\subsection{Robustness to Different Architectures}
\setlength{\tabcolsep}{6pt}
\begin{table}[!htbp]
  \centering
  \caption{
  \textbf{Robustness of different architectures on 11 datasets.} 
  The models are trained under the 16-shot setting with different visual architectures of CLIP.
  \textbf{Bold} denotes the highest results.
}
  \resizebox{1.0\textwidth}{!}{
    \begin{tabular}{llcccccccccccc}
    \toprule
    \multicolumn{2}{c}{\textbf{Method}} & \textbf{Pets} & \textbf{Flowers} & \textbf{FGVC} & \textbf{DTD} & \textbf{EuroSAT} & \textbf{Cars} & \textbf{Food} & \textbf{SUN} & \textbf{Cal} & \textbf{UCF} & \textbf{IN} & \textbf{Avg.} \\
    \midrule
    \multicolumn{2}{l}{\textbf{ResNet-50}} & \multicolumn{12}{c}{} \\
    \multicolumn{2}{l}{Zero-Shot CLIP} & 85.77  & 66.14  & 17.28  & 42.32  & 37.56  & 55.61  & 77.31  & 58.52  & 86.29  & 61.46  & 58.18  & 58.77  \\
    \multicolumn{2}{l}{CoOp} & 87.01  & 94.51  & 31.26  & 63.58  & 83.53  & 73.36  & 74.67  & 69.26  & 91.83  & 75.71  & 62.95  & 73.42  \\
    \multicolumn{2}{l}{Tip-Adapter} & 88.14  & 89.89  & 29.76  & 60.93  & 70.54  & 66.77  & 77.83  & 66.85  & 90.18  & 70.58  & 62.01  & 70.32  \\
    \rowcolor{gray!40}  
    \multicolumn{2}{l}{Ours} & \textbf{88.81}  & \textbf{95.72}  & \textbf{40.61}  & \textbf{66.51}  & \textbf{86.12}  & \textbf{75.12}  & \textbf{79.05}  & \textbf{70.70}  & \textbf{92.55}  & \textbf{77.53}  & \textbf{63.82}  & \textbf{76.05}  \\
    \midrule
    \multicolumn{2}{l}{\textbf{ResNet-101}} & \multicolumn{12}{c}{} \\
    \multicolumn{2}{l}{Zero-Shot CLIP} & 86.75  & 64.03  & 18.42  & 38.59  & 32.59  & 66.23  & 80.53  & 58.96  & 89.78  & 60.96  & 61.62  & 59.86  \\
    \multicolumn{2}{l}{CoOp} & 88.57  & 95.19  & 34.76  & 65.47  & 83.54  & 79.74  & 79.08  & 71.19  & 93.42  & 77.95  & \textbf{66.60}  & 75.96  \\
    \multicolumn{2}{l}{Tip-Adapter} & 87.23  & 90.77  & 31.51  & 62.37  & 66.45  & 72.96  & 81.31  & 67.96  & 93.01  & 73.53  & 64.41  & 71.96  \\
    \rowcolor{gray!40}  
    \multicolumn{2}{l}{Ours} & \textbf{91.43}  & \textbf{96.17}  & \textbf{42.58}  & \textbf{68.62}  & \textbf{86.32}  & \textbf{79.99}  & \textbf{82.15}  & \textbf{72.07}  & \textbf{93.63}  & \textbf{79.31}  & 66.33  & \textbf{78.06}  \\
    \midrule
    \multicolumn{2}{l}{\textbf{ViT-B/32}} & \multicolumn{12}{c}{} \\
    \multicolumn{2}{l}{Zero-Shot CLIP} & 87.49  & 66.95  & 19.23  & 43.97  & 45.19  & 60.55  & 80.50  & 61.91  & 90.87  & 62.01  & 62.05  & 61.88  \\
    \multicolumn{2}{l}{CoOp} & 88.68  & 94.97  & 33.22  & 65.37  & 83.43  & 76.08  & 78.45  & 72.38  & \textbf{94.62}  & 78.66  & 66.85  & 75.70  \\
    \multicolumn{2}{l}{Tip-Adapter} & 88.34  & 91.61  & 30.92  & 61.90  & 69.53  & 69.59  & 80.94  & 70.27  & 93.85  & 73.74  & 65.41  & 72.37  \\
    \rowcolor{gray!40}  
    \multicolumn{2}{l}{Ours} & \textbf{91.21}  & \textbf{96.16}  & \textbf{41.74}  & \textbf{67.63}  & \textbf{87.30}  & \textbf{77.55}  & \textbf{81.84}  & \textbf{73.60}  & 94.42  & \textbf{80.17}  & \textbf{67.00}  & \textbf{78.06}  \\
    \midrule
    \multicolumn{2}{l}{\textbf{ViT-B/16}} & \multicolumn{12}{c}{} \\
    \multicolumn{2}{l}{Zero-Shot CLIP} & 89.21  & 71.34  & 24.72  & 44.39  & 47.60  & 65.32  & 86.06  & 62.50  & 92.94  & 66.75  & 66.73  & 65.23  \\
    \multicolumn{2}{l}{CoOp} & 92.53  & 96.47  & 42.91  & 68.50  & 80.87  & \textbf{83.09}  & 87.21  & 75.29  & 95.77  & 82.24  & 71.92  & 79.71  \\
    \multicolumn{2}{l}{Tip-Adapter} & 91.54  & 94.41  & 39.48  & 65.68  & 76.58  & 75.44  & 86.47  & 71.85  & 95.10  & 77.94 & 70.46  & 76.81  \\
    \rowcolor{gray!40}  
    \multicolumn{2}{l}{Ours} & \textbf{93.73}  & \textbf{97.92}  & \textbf{50.33}  & \textbf{71.26}  & \textbf{89.19}  & 82.63  & \textbf{87.27}  & \textbf{75.87}  & \textbf{95.79}  & \textbf{84.09}  & \textbf{72.24}  & \textbf{81.85}  \\
    \bottomrule
    \vspace{-20pt}
    \end{tabular}%
}
  \label{tab:architecture}%
\end{table}%

We further evaluate the efficacy of our proposed method across 11 datasets with varying visual architectures of CLIP.
We selected two approaches for comparison: a training-required method, CoOp~\citep{zhou2022learning}, and a training-free method, Tip-Adapter~\citep{zhang2022tip}.
And these methods are trained on the 16-shot dataset.
As shown in Table~\ref{tab:architecture}, our method yielded a substantial improvement of 17.28\%, 18.20\%, 16.18\%, and 16.62\% on average, compared to the Zero-Shot CLIP~\citep{radford2021learning} approach, for ResNet-50, ResNet-101, ViT-B/32, and ViT-B/16 CLIP, respectively, across all 11 datasets.
The results demonstrate the effectiveness of our method across different CLIP architectures.

\subsection{Ablation of the Hyper-parameter $\alpha$}
As shown in Equation~(\ref{eq:ensemble}), our method needs a hyper-parameter $\alpha$ to integrate the knowledge from visual and text modalities. 
Specifically, we only performed a coarse search for $\alpha$ within [0.001, 0.01, 0.1, 1, 10, 100]. 
The search only ascertained the order of magnitude for alpha, providing a foundational understanding of its impact. 
The optimal alpha values resulting from this exploration are detailed in Table~\ref{tab:best_alpha}.
Notably, our analysis reveals a consistent concentration of optimal alpha values in the range of 1 to 10 across the majority of datasets. Subsequently, we conducted additional sensitivity experiments within this narrowed range.
The results of different $\alpha$ values are reported in Table~\ref{tab:ablation_alpha}.
The findings indicate a robust model sensitivity to alpha values, with the exception of $\alpha = 1$. 
And on average over 11 datasets, the model achieves performance around 75.5.

\setlength{\tabcolsep}{6pt}
\begin{table}[htbp]
  \centering
  \caption{
  The best $\alpha$ in our method for 16-shot datasets with RN50 CLIP. The best alpha is searched on [0.001, 0.01, 0.1, 1, 10, 100] using the validation set.
  }
  \resizebox{1.0\textwidth}{!}{
    \begin{tabular}{ccccccccccccc}
    \toprule
    \multicolumn{2}{c|}{} & \textbf{Pets} & \textbf{Flo} & \textbf{FGVC} & \textbf{DTD} & \textbf{EuroSAT} & \textbf{Cars} & \textbf{Food} & \textbf{SUN} & \textbf{Cal} & \textbf{UCF} & \textbf{IN} \\
    \midrule
    \multicolumn{2}{c|}{best $\alpha$} & 1.00  & 10.00  & 100.00  & 10.00  & 10.00  & 10.00  & 1.00  & 10.00  & 1.00  & 10.00  & 10.00  \\
    \multicolumn{2}{c|}{Ours} & 88.81  & 95.72  & 40.61  & 66.51  & 86.12  & 75.12  & 79.05  & 70.70  & 92.55  & 77.53  & 63.82   \\
    \bottomrule
    \end{tabular}%
    }
  \label{tab:best_alpha}%
\end{table}%

\setlength{\tabcolsep}{6pt}
\begin{table}[htbp]
  \centering
  \caption{
  \textbf{Ablation on hyper-parameter $\alpha$ of our method.} The models are trained under the 16-shot setting with RN50 CLIP.
  }
  \resizebox{1.0\textwidth}{!}{
    \begin{tabular}{cc|cccccccccccc}
    \toprule
    \multicolumn{2}{c|}{} & \textbf{Pets} & \textbf{Flo} & \textbf{FGVC} & \textbf{DTD} & \textbf{EuroSAT} & \textbf{Cars} & \textbf{Food} & \textbf{SUN} & \textbf{Cal} & \textbf{UCF} & \textbf{IN} & \textbf{Avg} \\
    \midrule
    \multicolumn{2}{l|}{$\alpha=1$} & 88.81  & 89.32  & 31.10  & 67.65  & 85.87  & 67.54  & 79.05  & 66.73  & 92.55  & 77.53  & 61.79  & 73.45  \\
    \multicolumn{2}{l|}{$\alpha=3$} & 87.35  & 95.18  & 38.20  & 67.42  & 85.86  & 74.07  & 77.80  & 70.15  & 92.63  & 78.53  & 63.19  & 75.49  \\
    \multicolumn{2}{l|}{$\alpha=5$} & 86.45  & 95.67  & 39.72  & 67.18  & 85.85  & 75.45  & 76.58  & 70.81  & 90.13  & 78.24  & 63.73  & 75.44  \\
    \multicolumn{2}{l|}{$\alpha=7$} & 85.71  & 95.90  & 40.36  & 67.00  & 85.82  & 75.82  & 75.76  & 70.97  & 91.97  & 77.84  & 63.91  & 75.55  \\
    \multicolumn{2}{l|}{$\alpha=9$} & 85.36  & 95.99  & 40.63  & 66.86  & 85.80  & 75.62  & 75.17  & 70.79  & 91.94  & 77.59  & 63.88  & 75.42  \\
    \multicolumn{2}{l|}{$\alpha=10$} & 85.26  & 95.72  & 40.69  & 66.51  & 86.12  & 75.12  & 74.94  & 70.70  & 91.90  & 77.53  & 63.82  & 75.30  \\
    \bottomrule
    \end{tabular}%
    }
  \label{tab:ablation_alpha}%
\end{table}%

\section{Experimental Details}
\label{appendix:b}
\subsection{Statistic of Datasets}
Following previous work~\citep{zhou2022conditional, zhou2022learning, wang2023improving, huang2022unsupervised, wang2023exploring}, we conduct experiments on 17 publicly available image classification datasets.
The datasets include ImageNet~\citep{deng2009imagenet}, Caltech101~\citep{fei2004learning}, OxfordPets~\citep{parkhi2012cats}, StanfordCars~\citep{krause20133d}, Flowers102~\citep{nilsback2008automated}, Food101~\citep{bossard2014food}, FGVCAircraft~\citep{maji2013fine}, EuroSAT~\citep{helber2019eurosat}, UCF101~\citep{soomro2012ucf101}, DTD~\citep{cimpoi2014describing}, SUN397~\citep{xiao2010sun}, ImageNetV2~\citep{recht2019imagenet},
ImageNet-Sketch~\citep{wang2019learning}, ImageNet-A~\citep{hendrycks2021natural}, ImageNet-R~\citep{hendrycks2021many}, ImageNet-LT~\citep{liu2019large}, and Places-LT~\citep{zhou2017places}.

\begin{table}[htbp]
  \centering
  \caption{Detailed statistics of datasets used in experiments.}
    \begin{tabular}{lrrrcc}
    \toprule
    Dataset & \# Classes & \# Training & \# Test & \multicolumn{2}{c}{Task} \\
    \midrule
    {OxfordPets} & 37    & 2,944 & 3,669 & \multicolumn{2}{c}{fine-grained pets recognition} \\
    {Flowers102} & 102   & 4,093 & 2,463 & \multicolumn{2}{c}{fine-grained flowers recognition} \\
    {FGVCAircraft} & 100   & 3,334 & 3,333 & \multicolumn{2}{c}{fine-grained aircraft recognition} \\
    {DTD} & 47    & 2,820 & 1,692 & \multicolumn{2}{c}{Textural recognition} \\
    {EuroSAT} & 10    & 13,500 & 8,100 & \multicolumn{2}{c}{Satellite image recognition} \\
    {StanfordCars} & 196   & 6,509 & 8,041 & \multicolumn{2}{c}{Fine-grained car recognition} \\
    {Food101} & 101   & 50,500 & 30,300 & \multicolumn{2}{c}{Fine-grained food recognition} \\
    {Sun397} & 397   & 15,880 & 19,850 & \multicolumn{2}{c}{Scene recognition} \\
    {Caltech101} & 100   & 4,128 & 2,465 & \multicolumn{2}{c}{Object recognition} \\
    {UCF101} & 101   & 7,639 & 3,783 & \multicolumn{2}{c}{Action recognition} \\
    {ImageNet} & 1,000 & 1.28M & 50,000 & \multicolumn{2}{c}{Object recognition} \\
    \midrule
    {ImageNetV2} & 1,000 & -     & 10,000 & \multicolumn{2}{c}{Robustness of collocation} \\
    {ImageNet-Sketch} & 1,000 & -     & 50,889 & \multicolumn{2}{c}{Robustness of sketch domain} \\
    {ImageNet-A} & 200   & -     & 7,500 & \multicolumn{2}{c}{Robustness of adversarial} \\
    {ImageNet-R} & 200   & -     & 30,000 & \multicolumn{2}{c}{Robustness of rendition styles} \\
    \midrule
    {ImageNet-LT} & 1,000 & 115,846 & 50,000 & \multicolumn{2}{c}{long-tail object recognition} \\
    {Places-LT} & 365   & 62,500 & 7300  & \multicolumn{2}{c}{long-tail place recognition} \\
    \bottomrule
    \end{tabular}%
  \label{tab:addlabel}%
\end{table}%

\subsection{Prompt Templates for Each Dataset}
For the zero-shot classifier, we employ handcrafted prompts to generate the classifier weight, as proposed in CLIP~\citep{radford2021learning}.
By default, we utilize the prompt template ``a photo of \{class\}." for class labels, where \{class\} represents the name of the classes. 
However, for fine-grained classification datasets such as FGVCAircraft~\citep{maji2013fine}, we incorporate the name of the superclass or a description into the template.
The prompt templates for each dataset are shown as follows.

\begin{table}[htbp]
  \centering
  \caption{Prompt templates for each class.}
    \begin{tabular}{cc|ccccc}
    \toprule
    \multicolumn{2}{c|}{\textbf{Dataset}} & \multicolumn{5}{c}{\textbf{Prompt template}} \\
    \midrule
    \multicolumn{2}{l|}{\multirow{3}[2]{*}{Caltech101~\citep{fei2004learning}}} & \multicolumn{5}{c}{\multirow{3}[2]{*}{``a photo of a \{class\}."}} \\
    \multicolumn{2}{c|}{} & \multicolumn{5}{c}{} \\
    \multicolumn{2}{c|}{} & \multicolumn{5}{c}{} \\
    \midrule
    \multicolumn{2}{l|}{\multirow{3}[2]{*}{OxfordPets~\citep{parkhi2012cats}}} & \multicolumn{5}{c}{\multirow{3}[2]{*}{``a photo of a \{class\}, a type of pet."}} \\
    \multicolumn{2}{c|}{} & \multicolumn{5}{c}{} \\
    \multicolumn{2}{c|}{} & \multicolumn{5}{c}{} \\
    \midrule
    \multicolumn{2}{l|}{\multirow{3}[2]{*}{StanfordCars~\citep{krause20133d}}} & \multicolumn{5}{c}{\multirow{3}[2]{*}{``a photo of a \{class\}."}} \\
    \multicolumn{2}{c|}{} & \multicolumn{5}{c}{} \\
    \multicolumn{2}{c|}{} & \multicolumn{5}{c}{} \\
    \midrule
    \multicolumn{2}{l|}{\multirow{3}[2]{*}{Flowers102~\citep{nilsback2008automated}}} & \multicolumn{5}{c}{\multirow{3}[2]{*}{``a photo of a \{class\}, a type of flower."}} \\
    \multicolumn{2}{c|}{} & \multicolumn{5}{c}{} \\
    \multicolumn{2}{c|}{} & \multicolumn{5}{c}{} \\
    \midrule
    \multicolumn{2}{l|}{\multirow{3}[2]{*}{Food101~\citep{bossard2014food}}} & \multicolumn{5}{c}{\multirow{3}[2]{*}{``a photo of \{class\}, a type of food."}} \\
    \multicolumn{2}{c|}{} & \multicolumn{5}{c}{} \\
    \multicolumn{2}{c|}{} & \multicolumn{5}{c}{} \\
    \midrule
    \multicolumn{2}{l|}{\multirow{3}[2]{*}{FGVCAircraft~\citep{maji2013fine}}} & \multicolumn{5}{c}{\multirow{3}[2]{*}{``a photo of a \{class\}, a type of aircraft."}} \\
    \multicolumn{2}{c|}{} & \multicolumn{5}{c}{} \\
    \multicolumn{2}{c|}{} & \multicolumn{5}{c}{} \\
    \midrule
    \multicolumn{2}{l|}{\multirow{3}[2]{*}{SUN397~\citep{xiao2010sun}}} & \multicolumn{5}{c}{\multirow{3}[2]{*}{``a photo of a \{class\}."}} \\
    \multicolumn{2}{c|}{} & \multicolumn{5}{c}{} \\
    \multicolumn{2}{c|}{} & \multicolumn{5}{c}{} \\
    \midrule
    \multicolumn{2}{l|}{\multirow{3}[2]{*}{DTD~\citep{cimpoi2014describing}}} & \multicolumn{5}{c}{\multirow{3}[2]{*}{``\{class\} texture."}} \\
    \multicolumn{2}{c|}{} & \multicolumn{5}{c}{} \\
    \multicolumn{2}{c|}{} & \multicolumn{5}{c}{} \\
    \midrule
    \multicolumn{2}{l|}{\multirow{3}[2]{*}{EuroSAT~\citep{helber2019eurosat}}} & \multicolumn{5}{c}{\multirow{3}[2]{*}{``a centered satellite photo of \{class\}."}} \\
    \multicolumn{2}{c|}{} & \multicolumn{5}{c}{} \\
    \multicolumn{2}{c|}{} & \multicolumn{5}{c}{} \\
    \midrule
    \multicolumn{2}{l|}{\multirow{3}[2]{*}{UCF101~\citep{soomro2012ucf101}}} & \multicolumn{5}{c}{\multirow{3}[2]{*}{``a photo of a person doing \{class\}."}} \\
    \multicolumn{2}{c|}{} & \multicolumn{5}{c}{} \\
    \multicolumn{2}{c|}{} & \multicolumn{5}{c}{} \\
    \midrule
    \multicolumn{2}{l|}{\multirow{6}[2]{*}{ImageNet~\citep{deng2009imagenet}}} & \multicolumn{5}{c}{``a bad photo of the \{class\}."} \\
    \multicolumn{2}{c|}{} & \multicolumn{5}{c}{``a origami \{class\}."} \\
    \multicolumn{2}{c|}{} & \multicolumn{5}{c}{``a photo of the large \{class\}."} \\
    \multicolumn{2}{c|}{} & \multicolumn{5}{c}{``a \{class\} in a video game."} \\
    \multicolumn{2}{c|}{} & \multicolumn{5}{c}{``art of the \{class\}."} \\
    \multicolumn{2}{c|}{} & \multicolumn{5}{c}{``a photo of the small \{class\}."} \\
    \midrule
    \multicolumn{2}{l|}{\multirow{3}[2]{*}{ImageNet-LT~\citep{liu2019large}}} & \multicolumn{5}{c}{\multirow{3}[2]{*}{``a photo of a \{class\}."}} \\
    \multicolumn{2}{c|}{} & \multicolumn{5}{c}{} \\
    \multicolumn{2}{c|}{} & \multicolumn{5}{c}{} \\
    \midrule
    \multicolumn{2}{l|}{\multirow{3}[2]{*}{Places-LT~\citep{zhou2017places}}} & \multicolumn{5}{c}{\multirow{3}[2]{*}{``a photo of a \{class\}."}} \\
    \multicolumn{2}{c|}{} & \multicolumn{5}{c}{} \\
    \multicolumn{2}{c|}{} & \multicolumn{5}{c}{} \\
    \bottomrule
    \end{tabular}%
  \label{tab:template}%
\end{table}%
\end{document}